\documentclass[11pt]{article}
\usepackage[left=1in, right=1in, top=1in]{geometry}
\usepackage{xargs}
\usepackage[numbers]{natbib}
\usepackage{fancyhdr}
\usepackage{setspace}
\usepackage{lastpage}
\usepackage{upgreek}
\usepackage{amsmath,mathtools,amsthm,amsfonts,amssymb}	
\usepackage{amssymb,dsfont,bbm}	
\usepackage{xcolor}  
\usepackage{bm}
\usepackage{enumitem}
\usepackage{mathrsfs}

\usepackage{algorithm}
\usepackage{algorithmic}

\setcitestyle{number}

\usepackage[colorlinks=true,breaklinks=true,bookmarks=true,urlcolor=blue,citecolor=blue,linkcolor=blue,bookmarksopen=false,draft=false]{hyperref}

\usepackage{aliascnt}
\usepackage{cleveref}

\usepackage{microtype}
\usepackage{graphicx}
\usepackage{subcaption}
\usepackage{booktabs} 
\usepackage{times}

\usepackage{mathtools}
\usepackage{xargs}
\usepackage{bbm}
\usepackage{enumitem}

\usepackage{hyperref}
\mathtoolsset{showonlyrefs}

\usepackage[textsize=tiny,textwidth=2.0cm]{todonotes}
\newcommand{\cA}{\mathcal{A}}

\newcommand{\cN}{\mathcal{N}}

\newcommand{\cS}{\mathcal{S}}

\newcommand{\bR}{\mathbb{R}}

\newcommand{\ust}{^{\star}}

\newcommand{\eps}{\epsilon}

\newcommand{\br}[1]{\left(#1\right)}

\newcommand{\mycomment}[1]{}

\newtheorem{assumA}{\textbf{A}\hspace{-1pt}}
\Crefname{assumA}{\textbf{A}\hspace{-1pt}}{\textbf{A}\hspace{-1pt}}
\crefname{assumA}{\textbf{A}}{\textbf{A}}

\Crefname{assumM}{\textbf{M}\hspace{-1pt}}{\textbf{M}\hspace{-1pt}}
\crefname{assumM}{\textbf{M}}{\textbf{M}}

\newtheorem{theorem}{Theorem}
\crefname{theorem}{theorem}{Theorems}
\Crefname{Theorem}{Theorem}{Theorems}

\newtheorem{lemma}{Lemma}
\crefname{lemma}{lemma}{lemmas}
\Crefname{Lemma}{Lemma}{Lemmas}

\crefname{corollary}{corollary}{corollaries}
\Crefname{Corollary}{Corollary}{Corollaries}

\newaliascnt{proposition}{theorem}
\newtheorem{proposition}[proposition]{Proposition}
\aliascntresetthe{proposition}
\crefname{proposition}{proposition}{propositions}
\Crefname{Proposition}{Proposition}{Propositions}

\newaliascnt{definition}{theorem}

\aliascntresetthe{definition}
\crefname{definition}{definition}{definitions}
\Crefname{Definition}{Definition}{Definitions}

\newaliascnt{definitionProposition}{theorem}

\aliascntresetthe{definitionProposition}
\crefname{Proposition and Definition}{Proposition and Definition}{Proposition and Definition}
\Crefname{Proposition and Definition}{Proposition and Definition}{Proposition and Definition}

\newtheorem{remark}{Remark}
\crefname{remark}{remark}{remarks}
\Crefname{Remark}{Remark}{Remarks}

\crefname{example}{example}{examples}
\Crefname{Example}{Example}{Examples}

\crefname{figure}{figure}{figures}
\Crefname{Figure}{Figure}{Figures}

\newcommand{\E}{\mathsf{E}}
\newcommand{\Pb}{\mathsf{P}}

\newcommand{\eqsp}{\;}

\newcommand{\e}{\mathbf{e}}
\newcommand{\Lam}{\mathbf{\Lambda}}
\newcommand{\bMKQ}{\mathbf{P}}
\newcommand{\MKQ}{\mathrm{P}}

\newcommand{\totMKQ}{\bar{\mathrm{P}}}
\newcommand{\beps}{\boldsymbol{\varepsilon}}
\def\nset{\ensuremath{\mathbb{N}}}
\newcommand{\pois}[2]{\boldsymbol{g}^{#1}_{#2}}
\def\taumix{t_{\operatorname{mix}}}
\def\tvdist{\mathsf{d}_{\operatorname{tv}}}

\def\nset{\ensuremath{\mathbb{N}}}
\def\rset{\mathbb{R}}
\def\Xset{\mathsf{X}}
\def\Xsigma{\mathcal{X}}
\newcommand{\rmd}{\mathrm{d}}
\newcommand{\Sigmabf}{\boldsymbol{\Sigma}}

\newcommand{\wR}{\widehat{\mathrm{R}}}
\newcommand{\Cxi}{\operatorname{C}_{\xi}}
\newcommand{\Com}{\operatorname{C}_{\omega}}
\newcommand{\Cmoment}{\operatorname{C}_{\mathrm{M}}}

\newcommand{\moment}[2]{\mathrm{M}_{#1}^{#2}}
\newcommand{\gprod}[2]{\mathrm{g}_{#1:#2}}
\newcommand{\Rlast}{\mathrm{R}^{\operatorname{last}}}
\newcommand{\Rpr}{\mathrm{R}^{\operatorname{pr}}}
\newcommand{\dconv}{\mathbf{d}_{\mathrm{C}}}

\def\PE{\mathsf{E}}
\usepackage{times}

\begin{document}
\title{On Gaussian approximation for entropy-regularized Q-learning with function approximation}

\author{A. Rubtsov~\footnote{HSE University, Moscow, Russia,  \texttt{asrubtsov@hse.ru}},  R. Singh~\footnote{Mohamed Bin Zayed University of AI, UAE, \texttt{rahulsingh0188@gmail.com}},  E. Moulines~\footnote{Mohamed Bin Zayed University of AI, UAE, \texttt{eric.moulines@mbzuai.ac.ae}}
\footnote{EPITA, France, \texttt{eric.moulines@epita.fr}},  A. Naumov~\footnote{HSE University \& Steklov Mathematical Institute of Russian Academy of Sciences, Moscow, Russia,  \texttt{anaumov@hse.ru}.},
S. Samsonov~\footnote{HSE University, Moscow, Russia,  \texttt{svsamsonov@hse.ru}.}}

\maketitle

\begin{abstract}
In this paper, we derive rates of convergence in the high-dimensional central limit theorem for Polyak--Ruppert averaged iterates generated by entropy-regularized asynchronous Q-learning with linear function approximation and a polynomial stepsize $k^{-\omega}$, $\omega \in (1/2,1)$. Assuming that the sequence of observed triples $(s_k,a_k,s_{k+1})_{k \geq 0}$ forms a uniformly geometrically ergodic Markov chain, and under suitable regularity conditions for the projected soft Bellman equation, we establish a Gaussian approximation bound in the convex distance with rate of order $n^{-1/4}$, up to polylogarithmic factors in $n$, where $n$ is the number of samples used by the algorithm. To obtain this result, we combine a linearization of the soft Bellman recursion with a Gaussian approximation for the leading martingale term. Finally, we derive high-order moment bounds for the algorithm's last iterate, which might be of independent interest.
\end{abstract}

\section{Introduction}
Q-learning \cite{watkins1992qlearning} is one of the foundational model-free algorithms in reinforcement learning (RL), which proved its practical utility in various settings \cite{sutton2018reinforcement,mnih2015human}. At the same time, statistical analysis of Q-learning remains a challenging and delicate problem \cite{jin2018q,wainwright2019variance,qu2020finite}.

A substantial literature now explains when Q-learning converges and how fast its estimation error decays. Classical asymptotic analyses have been complemented by modern nonasymptotic rate and sample-complexity results in both synchronous and asynchronous settings~\cite{tsitsiklis1994asynchronous,qu2020finite,li2024minimax}. Most existing concentration analyses quantify only the magnitude of the error $\|Q_k - Q\ust\|$ through expectation or high-probability bounds. Yet for statistical tasks such as uncertainty quantification, one needs a finite-time central limit theorem (CLT), namely a bound that quantifies how accurately a normalized Q-learning error is approximated by a Gaussian distribution.

The question of the quality of Gaussian approximation for stochastic approximation algorithms recently gained a lot of attention \cite{shao2022berry, samsonov2024gaussian}, in particular for the Polyak--Ruppert averaged \cite{ruppert1988efficient,polyak1992acceleration} iterations. This technique is known to improve statistical efficiency while stabilizing and often accelerating stochastic approximation algorithms. For tabular $Q$-learning, the averaged iterates $\bar{Q}_n = \sum_{i=1}^{n}Q_i \slash n$ have been shown to be asymptotically normal, i.e. 
\[
\textstyle 
\sqrt{n} \br{\bar{Q}_n - Q\ust} \xrightarrow{d} \cN\br{0,\Sigma_{\infty}}\,,
\]
where $\Sigma_{\infty}$ is the limiting covariance~\cite{li2023statistical}. Recently, ~\cite{rubtsov2026gaussian} derived a Berry-Esseen theorem for tabular Q-learning which quantifies the rate of convergence to Gaussianity over the class of hyper-rectangles with convergence rate $O\br{\log(|\cS||\cA|n)n^{-1\slash 6}}$. At the same time, the latter result is accompanied with instance-dependent factors which unavoidably scale at least linearly with $|\cS||\cA|$. Such dependence may be tolerable in small tabular problems, but it becomes restrictive in the large-scale settings. A common technique to overcome this problem is to resort to function approximation techniques \cite{bertsekas1996neuro}.
\par 
Another strong assumption arising in most of the analysis of vanilla $Q$-learning is the uniqueness of the optimal policy $\pi^{\star}$, see e.g. \cite{li2023statistical} and the discussion after \Cref{thm:clt} in this work. This assumption is naturally satisfied in the setting of entropy-regularized $Q$-learning \cite{haarnoja2017reinforcement,haarnoja2018soft}. The latter algorithm has shown practical gains and, at the same time, the entropy regularization is a theoretically attractive feature which prevents the optimal policy from becoming degenerate. At the same time, statistical analysis of this algorithm remains limited with only an asymptotic CLT available \cite{li2023statistical}. This motivates the central question of the present paper: 
\vspace{-2mm}
\begin{quote}
\itshape
Can one establish a Gaussian approximation bound for entropy-regularized $Q$-learning with linear function approximation whose finite-sample guarantees depend on intrinsic feature dimension rather than on the tabular cardinalities $S$ and $A$?
\end{quote}
\vspace{-2mm}
\textbf{Our main contribution} is an affirmative answer to this question. Precisely, we derive a Gaussian approximation for entropy-regularized $Q$-learning with linear function approximation in convex distance (see formal definition in \Cref{subsec:gaus_approx}, equation \eqref{eq:berry-esseen}) with the rate $\tilde{O}(n^{-1/4})$\footnote{We use the notation \(\tilde{O}\) to denote bounds up to polylogarithmic factors.} up to polylogarithmic factors in $n$. The quality of Gaussian approximation is determined by the sample size $n$ and feature dimensionality $d$ rather than the size of the ambient space. We show that the entropy-regularized problem admits faster convergence rate compared to the existing results for vanilla $Q$-learning \cite{rubtsov2026gaussian}. 
\vspace{-2mm}
\subsection{Related works}
\vspace{-2mm}
Asymptotic convergence of Q-learning and SARSA algorithms with linear function approximation under appropriate assumptions was shown in~\cite{melo2008analysis}. More recently,~\cite{zou2019finite} performs a finite-time analysis of Q-learning and SARSA under the same set of assumptions and shows that the mean square error decays as $\tilde{O}\br{1/n}$, while~\cite{chen2022finite} derives similar results under a different set of assumptions. \cite{xie2022statistical} establishes a functional CLT (FCLT) for SA with constant step-sizes ($Q$-learning with fixed step-size is one application of this), and constructs confidence intervals for parameters via random scaling. \cite{li2023statistical} studies Q-learning for MDPs under the generative assumption, and derives asymptotic CLT for entropy-regularized Q-learning. 
\par 
The mentioned Gaussian approximation results, despite their statistical significance, remain fully asymptotic. Available non-asymptotic guarantees primarily arise for stochastic gradient descent (SGD) algorithm \cite{shao2022berry,sheshukova2025gaussian} and linear stochastic approximation (LSA)  \cite{samsonov2024gaussian,wu2024statistical} with independent noise. The latter LSA results were generalized to LSA with Markov noise under the standard uniform geometric ergodicity assumptions in \cite{wu2025uncertainty,samsonov2026statistical}. Non-asymptotic CLT for Q-learning has been recently developed in \cite{liu2025central} in the Wasserstein
distance at a rate of $O\big(\sqrt{SA} n^{-1\slash 6}\big)$. The latter polynomial dependence was improved in a recent work \cite{rubtsov2026gaussian}, where $\tilde{O}\br{n^{-1\slash 6}}$ convergence rate has been achieved for the class of hyper-rectangles. Detailed discussion of these two results is provided after \Cref{thm:clt}.

\vspace{-2mm}
\paragraph{Notations.} For a Markov kernel $\MKQ$ on $(\Xset,\Xsigma)$, and a measurable function $f: \Xset \to \rset$, we set $\MKQ f(x) = \int_{\Xset} f(y) \MKQ(x,\rmd y)$. 
Define total variation distance $\tvdist(\mu, \nu)$ for probability measures $\mu, \nu$: $
    \tvdist(\mu, \nu) = \tfrac{1}{2}\sup |\mu(f) - \nu(f)|$, where supremum is taken over all functions with $\|f\|_{\infty} \le 1$. The notations $\|\cdot\|$ and $\|\cdot\|_{\infty}$ refer to the spectral norm and the $\ell_\infty$ norm,
respectively. We use $\lesssim$ for inequalities that hold up to an
absolute, problem-independent constant, and \(\lesssim_{\mathrm{pr}}\) to denote
inequalities that hold up to a problem-dependent constant. We write
\(a\asymp b\) if both \(a\lesssim b\) and \(b\lesssim a\) hold.

\section{Main results}

\subsection{Q-learning with function approximation}

We begin this section by specifying the set of assumptions that will be used to derive moment bounds and a
non-asymptotic central limit theorem for Q-learning iterates.
Consider an infinite-horizon discounted Markov decision process (DMDP) defined by the tuple 
\(\mathcal{M} = (\cS, \cA, \MKQ, r, \gamma)\), where \(\cS\) denotes the state space, \(\cA\) denotes the action space, 
\(\MKQ : \cS \times \cA \to \Delta(\cS)\) represents the controlled transition probabilities, and \(\gamma \in (0,1)\) is the discount factor. 
Throughout, we impose the following regularity condition.

\begin{assumA}
\label{assum:regularity}
Assume that both \(\cS\) and \(\cA\) are finite sets with cardinalities \(S\) and \(A\), respectively. 
Moreover, the immediate reward function \(r(s,a)\) is deterministic and bounded within the interval \([0,1]\).
\end{assumA}

A stationary Markov \textit{policy} $\pi$ samples the action as a function of the current state, i.e. $\pi:\cS\to\Delta(\cA)$. Given an initial state $s_0\sim \nu$, a policy $\pi$, and the transition kernel $\MKQ$,
the process evolves as $a_t\sim \pi(\cdot\mid s_t)$ and $s_{t+1}\sim \MKQ(\cdot\mid s_t,a_t)$.~The entropy-regularized discounted
return is defined as 
\[
J_ \lambda^\pi
=
\sum_{t=0}^{\infty}\gamma^t
\big(
r(s_t,a_t)
+
 \lambda \mathcal{H}\big(\pi(\cdot\mid s_t)\big)
\big)\eqsp,
\]
where $\lambda>0$ is the temperature parameter and $\mathcal{H}\big(\pi(\cdot\mid s)\big) =
-\sum_{a\in\mathcal{A}}\pi(a\mid s)\log \pi(a\mid s)$ is the Shannon entropy of the action probabilities in state $s$. Adding the entropy term tends to improve exploration and often makes learning more robust to estimation errors~\cite{ziebart2008maximum,haarnoja2017reinforcement,haarnoja2018soft}.~The corresponding entropy-regularized value and action-value functions are
\[
V_ \lambda^\pi(s)
=
\E\big[J_ \lambda^\pi\mid s_0=s\big]\eqsp,\qquad Q_\lambda^\pi(s,a)
=
\E\big[J_\lambda^\pi\mid s_0=s,\ a_0=a\big]\eqsp .
\]
The goal is to learn the optimal
entropy-regularized action-value function
\[
Q_ \lambda^\star(s,a)
=
\sup_{\pi} Q_ \lambda^\pi(s,a)
\]
in the setting where the transition kernel $\MKQ$ is unknown. The Q-learning algorithm iteratively updates an approximation of the optimal action–value function 
based on observed data. Assume that the agent only has access to a single  trajectory \(\{s_t,a_t,r_t\}_{t=0}^{T}\) generated under a fixed \textit{behavior policy} \(\pi_b\). 
Formally, the data are generated according to
\[
a_t \sim \pi_b(\cdot \mid s_t), 
\quad r_t = r(s_t,a_t), 
\quad s_{t+1} \sim \MKQ(\cdot \mid s_t,a_t).
\]
We further identify $\MKQ$ with a stochastic matrix of shape $\rset^{SA \times S}$. For any stationary policy \(\pi\), we define the corresponding transition matrix by
\begin{equation}
    \MKQ^{\pi}\bigl((s,a),(s',a')\bigr)
    = \MKQ(s' \mid s,a)\,\pi(a' \mid s') \eqsp.
\end{equation}
The behavior policy induces an important sequence of observations \(  z_k = (s_k, a_k, s_{k+1}) \) on the space  
\(\mathcal{Z} = \mathcal{S} \times \mathcal{A} \times \mathcal{S} \).
For this chain, we denote by $\mu$ its stationary distribution, assuming that it exists and is unique. The associated Markov kernel writes as 
\[
\bar{\MKQ}((s_2,a_2,s_2') | (s_1,a_1,s_1')) = \mathbf{1}\{s_2 = s_1'\} \pi_b(a_2 | s_2) \MKQ(s_2' \mid s_2, a_2) \eqsp. 
\]
Sufficient exploration of the
transition dynamic is ensured by the following assumption.
\begin{assumA}
\label{assum:UGE}
The Markov kernel \(\bar{\mathrm{\MKQ}}\) is uniformly 
geometrically ergodic, that is, there exists \(\taumix \in \nset\) such that for all \(t \in \nset\) and \(z\in \mathcal{Z}\),
\begin{equation}
 d_{\operatorname{tv}}\big( \bar{\MKQ}^t(\,\cdot \mid z), {\mu} \bigr)
   \;\leq\; \left(1/4\right)^{\lfloor t / \taumix \rfloor}\, .
\end{equation}
In addition, the Markov chain is initialized from stationarity, i.e., \(z_0\sim\mu\).
\end{assumA}
\Cref{assum:UGE} is standard in finite-time analysis, both for last-iterate convergence \cite{chen2022finite,li2024minimax} and for CLT results \cite{rubtsov2026gaussian,liu2025central}. We note that the assumption \(z_0\sim\mu\) is introduced only for clarity of exposition. The same arguments extend without difficulty to arbitrary initial distributions.

We study Q-learning with linear function approximation where the Q-function can be represented by a linear combination of a set
of \(d\) feature functions \(\phi_1, \ldots, \phi_d\), where $\phi_i : \cS \times \cA \to \bR$. 
For a given \(\theta \in \rset^d\), we approximate the action-value function as \(Q_\theta(s,a) = \sum_{k=1}^d\phi_k(s,a)\theta_k\). With
the feature matrix \(\Phi=(\phi_1, \ldots, \phi_d)\in\rset^{SA\times d}\) we have \(Q_\theta = \Phi\theta\). For a temperature parameter \( \lambda>0\), 
the soft value function is defined as follows~\cite{ziebart2008maximum,haarnoja2017reinforcement,haarnoja2018soft}, 
\[
V_\theta^ \lambda(s)
:=
 \lambda \log \sum_{a\in\mathcal A}
\exp\big({Q_\theta(s,a)}/\lambda\big)\eqsp.
\]
We define the soft-greedy policy induced by \(Q_{\theta}\) as
\[
\pi_\theta^\lambda(a\mid s)
=
\frac{
\exp\bigl(Q_{\theta}(s,a)/\lambda\bigr)
}{
\sum_{a'\in\mathcal A}
\exp\bigl(Q_{\theta}(s,a')/\lambda\bigr)
}\eqsp .
\]

The goal is to find a solution to the projected Bellman equation
\begin{align}
\label{eq:projected_bellman}
\Pi_\mu \big( r + \gamma\MKQ V^{\lambda}_{\theta^\star} - Q_{\theta^\star} \big) = 0,
\end{align}
where $\Pi_\mu$ denotes the orthogonal projection onto $\mathrm{span}(\Phi)$
with respect to the inner product induced by measure \(\mu\).
Asynchronous entropy-regularized Q-learning with function approximation is
formalized as follows. The agent initializes \(\theta_0\) arbitrarily. At time
\(t\), the agent observes the transition tuple
\((s_t,a_t,r_t,s_{t+1})\) and updates \(\theta_t\) according to
\begin{align}\label{eq:q_learning_updates}
\theta_{t+1}
=
\theta_t
+
\alpha_t\,\phi(s_t,a_t)
\Big(
r(s_t,a_t)
+
\gamma V_{\theta_t}^\lambda(s_{t+1})
-
\phi(s_t,a_t)^\top\theta_t
\Big)\eqsp,
\end{align}
where \(\{\alpha_t\}_{t\ge 0}\) is a step-size sequence. 
The complete algorithm is described in Algorithm \ref{Q learn alg}. 
\begin{algorithm}
    \centering
    \caption{Q-learning with function approximation}\label{Q learn alg}
    \begin{algorithmic}\label{algo:Q-learn}
        \STATE { \textbf{Input parameters:} learning rates $\{\alpha_t\}$, number of iterations $T$.}
        \STATE \textbf{Initialization:} $\theta_0 = 0$.
        \FOR{$t = 0,1,\ldots,T-1$}
          \STATE Draw action $a_{t} \sim \pi_b(\,\cdot \mid s_{t})$
    \STATE Draw next state $s_{t+1} \sim P(\,\cdot \mid s_{t}, a_{t})$
    \STATE Update $\theta_{t+1}$ according to \eqref{eq:q_learning_updates}
        \ENDFOR
    \end{algorithmic}
\end{algorithm}

In this work, we consider polynomially decaying step sizes \(\{\alpha_t\}\). While the choice \(\omega = 1\) yields optimal mean-squared error rates for the last iterate, it is suboptimal for Gaussian approximation of Polyak--Ruppert averages. In particular, as shown in \Cref{thm:clt}, the optimal convergence rate in Gaussian approximation is achieved for \(\omega = 3/4\). This motivates restricting attention to polynomial step sizes with \(\omega \in (1/2,1)\). 
\begin{assumA}
\label{assum:steps}
Step sizes $\{\alpha_{t}\}_{t \in \nset}$ have a form $\alpha_t = \frac{c_0}{(t+k_0)^{\omega}}$, where $\omega\in (\frac{1}{2}, 1)$ and the initialization parameter $k_0$ satisfies $k_0^{1-\omega}\geq 32/(\kappa c_0)$. Moreover,  
\begin{equation}
\label{eq:k_0_lower_bound}
\alpha_0\lesssim  \frac{\kappa}{ p^2C_2\vee1}\eqsp,
\end{equation}
where \(\kappa\) is defined in \Cref{assum:features}, \(\Cxi\) in \Cref{lem:xi_t_bound}, and \(C_2\) in \eqref{eq:C_i_definitions}.
\end{assumA}
\begin{remark}
We note that \(\alpha_0\) scales as \(p^{-2}\). In \Cref{thm:clt}, we set the parameter $p$ of order \(\log(dn)\). Hence, the parameter $k_0$ will depend on the total number of iterations to be performed. The same effect appears in the single-timescale SA algorithms \cite{durmus2022finite,wu2024statistical} and two-timescale linear SA algorithms \cite{butyrin2026gaussian}. This effect is unavoidable at least in the setting of the constant step-size algorithms, see \cite[Theorem~1]{durmus2021tight}. 
\end{remark}
The main difficulty of Q-learning with function approximation is that the algorithm may diverge in the absence of additional assumptions linking the behavior policy \(\pi_b\) and the feature map \(\phi\); see, e.g., \cite{baird1995residual}. To address this issue, we introduce the covariance matrices
\begin{align}
\Sigma_{\phi}^{\pi_b}
=
\mathbb{E}_{\mu}
\bigl[
\phi(s,a)\phi(s,a)^\top
\bigr]\eqsp,\quad \Sigma_{\phi}^{\theta}
=
\mathbb{E}_{s\sim \mu,\; a\sim \pi_\theta^\lambda(\cdot\mid s)}
\bigl[
\phi(s,a)\phi(s,a)^\top
\bigr]\eqsp.
\end{align}
Throughout, the feature map is assumed to satisfy the following condition.

\begin{assumA}\label{assum:features}
The feature vectors are uniformly bounded, i.e., \(\|\phi(s,a)\|\le 1\) for all \((s,a)\). Moreover, the feature covariance under the behavior distribution dominates that induced by the soft-greedy policy, in the sense that for all \(\theta \in \rset^d\),
\[
\Sigma_{\phi}^{\pi_b}
-
\gamma^2 \Sigma_{\phi}^{\theta}
\succeq
\kappa I \eqsp.
\]
\end{assumA}
\Cref{assum:features} ensures that the projected Bellman equation \eqref{eq:projected_bellman} has a unique solution $\theta^{\star}$, see \cite[Theorem 4]{zou2019finite}. \Cref{assum:features} appears in a similar form in prior works on finite-time analysis with function approximation \cite{chen2022finite,melo2008analysis,zou2019finite}, with the only distinction being its adaptation here to the entropy-regularized setting. This assumption may be viewed as an analogue of strong convexity in stochastic gradient analyses \cite{sheshukova2025gaussian}; see \Cref{lem:strong_convexity}. In particular, it ensures that the Jacobian of the projected Bellman operator at \(\theta^\star\) is Hurwitz; see \Cref{lem:contraction}. Such conditions are standard in the analysis of nonlinear stochastic approximation schemes \cite{polyak1992acceleration, sheshukova2025gaussian}.

\subsection{Moment bounds for Q-learning}
As a crucial step for Gaussian approximation, we need to establish bound for high order moments for the algorithm's last iterate. We believe that this bound is of independent interest and extends previous result of \cite{chen2022finite} for the second moment. Define an operator \(F : \rset^d \times \cS \times \cA \times \cS \to \rset^d\) by the relation 
\begin{align}
    F(\theta, z) = \phi(s, a)\big(r(s,a) + \gamma V^{\lambda}_{\theta}(s') - Q_{\theta}(s,a)\big)\eqsp, \quad \text{ where } z = (s,a,s')\,.
\end{align}
Then \eqref{eq:q_learning_updates} can be written in the form 
\[
\theta_{t+1} = \theta_t + \alpha_t F(\theta_t, z_t)\,.
\]
Let $\bar{F}(\theta) = \PE_{\mu}[F(\theta, z)]$. The convergence analysis of \Cref{algo:Q-learn} heavily relies on the Lipschitzness of  \(F\)(see \Cref{lem:F_lipshitz}), and  strong monotonicity of \(\bar F\) (see \Cref{lem:strong_convexity}) under \Cref{assum:features}. We also mention that with entropy regularization operators \(F\) and \(\bar{F}\) are differentiable, with Lipschitz Jacobian (see \Cref{lem:jacobian_smoothness}), but this is needed only for GAR.
By \Cref{lem:projected_bellman_matrix_form}, \(\bar{F}(\theta^\star) = 0\) coincides with the target equation \eqref{eq:projected_bellman}.
In operator form, \eqref{eq:q_learning_updates} can be written as
\begin{align}
\Delta_{t+1}
&= \Delta_t + \alpha_t \bar{F}(\theta_t) + \alpha_t \xi_t\eqsp,
\end{align}
where \(\Delta_t = \theta_t - \theta^\star\) and the noise term is defined by
\begin{align}
\label{eq:xi_t_definition}
\xi_t = F(\theta^\star, z_t)
+ \bigl\{F(\theta_t, z_t) - F(\theta^\star, z_t)\bigr\}
+ \bigl\{\bar{F}(\theta^\star) - \bar{F}(\theta_t)\bigr\}.
\end{align}
Here we used \(\bar{F}(\theta^\star)=0\).
Define scalar products 
\begin{align}
\label{eq:gprod_definition}
    \gprod{m}{n} = \prod_{j=m}^n(1 - \kappa\alpha_j/4)\eqsp,
\end{align}
with the convention that empty products equal \(1\).
In our proof, we proceed with the framework based on the reduction of Markov chains to martingales via the Poisson equation, see \cite[Chapter 21]{douc:moulines:priouret:soulier:2018}. Specifically, we decompose the noise into two components:
\begin{align}
    \xi_t = \xi_t^{(0)} + \xi_t^{(1)}\eqsp.
\end{align}
The first component, \(\xi_t^{(0)}\), forms a martingale-difference sequence, whereas the second component, \(\xi_t^{(1)}\), is a "smaller" residual. Explicit expressions for \(\xi_t^{(0)}\) and \(\xi_t^{(1)}\) are provided in \eqref{eq:xi_1_2_definitions}. This decomposition naturally arises in the stochastic approximation literature, see e.g. \cite{kaledin2020finite}. 

The following lemma provides a recursive bound on the squared error in terms of a transient term, a martingale noise component, and a Markovian noise component. See \Cref{sec:proof_of_moments} for the proof.
\begin{lemma}
\label{lem:square_error_rec}
Assume that \Cref{assum:regularity}--\Cref{assum:features} hold. Then,
    for all \(t\geq 0\),
    \begin{align}
    \|\Delta_{t+1}\|^2 \leq  2\sum_{j=0}^t\alpha_j\gprod{j+1}{t}\big\langle \Delta_j, \xi_j^{(0)}\rangle + 2\sum_{j=0}^t\alpha_j\gprod{j+1}{t}\big\langle \Delta_j, \xi_j^{(1)}\rangle + \alpha_t\Big(\frac{8\Cxi}{\kappa} + \frac{\|\Delta_0\|^2}{\alpha_0}\Big)\eqsp,
    \end{align}
where \(\Cxi\) is defined in \Cref{lem:xi_t_bound}.
\end{lemma}
Denote the error moments by
\[
\moment{t}{p} = \PE^{2/p}\!\left[\|\Delta_t\|^{2p}\right]\eqsp.
\]
We note that \(\moment{t}{p}\) is a non-homogeneous fourth-order quantity. This definition is motivated by the recursive inequality in \Cref{lem:square_error_rec}.
We proceed as follows. Applying Minkowski's inequality to the result of \Cref{lem:square_error_rec}, we reduce the problem to estimating the moments of the martingale and Markovian components. For the martingale component, we apply the Burkholder inequality, while the moments of the Markovian term are controlled via algebraic manipulations combined with Minkowski's inequality. As a result, we obtain a recursive bound on \(\moment{t+1}{p}\) in terms of \(\{\moment{j}{p}\}_{j=0}^t\). Solving this recursion yields the desired bound.
\begin{theorem}
\label{thm:moments}
   Assume that \Cref{assum:regularity}--\Cref{assum:features} hold. Then, for any $t>0$ and $p \ge 2$, it holds that
\[\moment{t+1}{p} \le \Cmoment p^4\alpha_t^2\eqsp,\]
where \(\Cmoment\) is defined in \Cref{lem:union_recursion}.
\end{theorem}
\paragraph{Proof sketch.}
Applying Minkowski's inequality to the result of \Cref{lem:square_error_rec} and
substituting results from \Cref{lem:xi_0_bound} and \Cref{lem:xi_1_bound} into the above bound leads to the following recursion inequality
\begin{align}
    \moment{t+1}{p} \leq C_0(\alpha_t/\alpha_0)^2 +   C_1\alpha_t^2 + p^2C_2\sum_{j=0}^{t}\alpha_j^2\gprod{j+1}{t}^2 \Big( \sqrt{\moment{j}{p}} + {\moment{j}{p}}\Big)\eqsp,
\end{align}
where the explicit values of the problem-dependent constants \(C_0, C_1\) and \(C_2\) are provided in \eqref{eq:C_i_definitions}.
We prove the claim by induction. The base  of induction follows from the definition of \(\Cmoment\). Assume that 
\begin{align}
\moment{j}{p}\leq p^4\Cmoment\alpha_j^2
\end{align}
holds for all \(j \le t\). We now verify the bound for \(j=t+1\). Indeed,
\begin{align}
    \moment{t+1}{p} &\leq C_0(\alpha_t/\alpha_0)^2 +   C_1\alpha_t^2 + p^2C_2\sum_{j=0}^{t}\alpha_j^2\gprod{j+1}{t}^2 \Big( \sqrt{\moment{j}{p}} + {\moment{j}{p}}\Big)\\
    &\overset{(a)}{\le} C_0(\alpha_t/\alpha_0)^2 +  C_1\alpha_t^2 + p^4C_2\sqrt{\Cmoment}\sum_{j=0}^{t}\alpha_j^3\gprod{j+1}{t}^2  +  p^6C_2\Cmoment\sum_{j=0}^{t}\alpha_j^4\gprod{j+1}{t}^2 \\
    &\overset{(b)}{\leq} C_0(\alpha_t/\alpha_0)^2 +   C_1\alpha_t^2 +\frac{ 16p^4C_2\sqrt{\Cmoment}\alpha_t^2}{\kappa}  +  \frac{16p^6C_2\Cmoment\alpha_t^3}{\kappa}\eqsp,
\end{align}
where step \((a)\) uses induction assumption \eqref{eq:induction_assumption}, step \((b)\) uses \Cref{lem:rate_of_convergence}. To complete the induction step, we bound each term in
\eqref{eq:indiction_reducing} by  \(\frac{p^4\Cmoment}{4}\). This is guaranteed under \Cref{assum:steps} together with a suitable choice of \(\Cmoment\); see \Cref{lem:union_recursion}.
\begin{remark}
    Choosing \(\alpha_0\asymp \frac{\kappa}{p^2C_2}\) we obtain convergence rate 
\begin{align}
    \PE^{1/p}[\|\theta_t - \theta^{\star}\|^p] \lesssim \frac{p\max(1,\lambda^2)}{t^{\omega/2}}\left(\frac{\|\theta^\star\|^2\|\theta_0-\theta^\star\|\taumix \Com^{1/2}\log^2 A}{\kappa^{1/2}}\right)\eqsp.
\end{align}

\end{remark}
\paragraph{Discussion.}
The full proof of \Cref{thm:moments} is postponed to \Cref{sec:proof_of_moments}. 
An immediate consequence of the theorem is that the error \(\|\theta_t-\theta^\star\|\) admits subexponential tails uniformly over time. 
The dependence on the entropy regularization parameter is mild: the constant \(\Cmoment\) scales at most quadratically in \(\lambda\) when \(\lambda>1\), while in the practically relevant regime \(\lambda\le1\), the convergence rate remains unaffected by regularization.

\Cref{thm:moments} extends the bounds of \cite{chen2022finite}, where a similar result is established for \(p=2\). A bound for the second moment is also obtained in \cite{zou2019finite}, although their analysis relies on an additional projection step to stabilize the iterates. While such projections can simplify the analysis, they typically destroy the linear structure required for the CLT decomposition. The bound in \Cref{thm:moments} is also consistent with recent results for tabular Q-learning \cite{rubtsov2026gaussian}. However, in the tabular setting the iterates \(Q_t\) are naturally bounded due to the intrinsic structure of the algorithm. This property no longer holds under function approximation and necessitates a more refined analysis.
\subsection{Gaussian approximation for Q-learning}
\label{subsec:gaus_approx}
\paragraph{Error decomposition.} In the derivation of \Cref{thm:moments}, scalar notation is used. 
For the Gaussian approximation, it is more convenient to proceed with matrix notation. 
Accordingly, before presenting the main result, we provide a matrix representation of the updates \eqref{eq:q_learning_updates}. Define the operators 
\(\Lam : \cS \times \cA \to \mathbb{R}^{SA \times SA}\) and 
\(\bMKQ : \cS \times \cA\times \cS \to \mathbb{R}^{SA \times S}\) by
\begin{equation}
\label{eq:P_and_Lam}
    \Lam(s,a) = \e_{s,a} (\e_{s,a})^\top, 
    \qquad 
    \bMKQ(s,a,s') = \e_{s,a} (\e_{s'})^\top,
\end{equation}
where \( \e_{s,a} \) and \( \e_s \) denote the canonical basis vectors associated with the state–action pair \((s,a)\) and the state \(s\), respectively.
For brevity, we introduce the random matrices 
\(\Lam_t = \Lam(s_t,a_t)\) and \(\bMKQ_{t} = \bMKQ(s_t,a_t,s_{t+1})\), which are functions of the underlying Markov chain $ z_t = (s_t,a_t,s_{t+1})$. With this notation, the update rule \eqref{eq:q_learning_updates} can equivalently be written as follows,
\begin{align}\label{eq:update_rule}
\theta_{t+1}
&=
\theta_t
+
\alpha_t\Phi^\top
\left(
\Lam_t r
+
\gamma \bMKQ_t V_{\theta_t}^\lambda
-
\Lam_t Q_{\theta_t}
\right)\eqsp .
\end{align}
 It is straightforward to verify that expectations with respect to the stationary distribution \(\mu\) take the form $\E_{\mu}[\Lam] = D_\mu$, where $D_\mu \in \rset^{SA \times SA}$ is a diagonal matrix of the form $D_\mu = \operatorname{diag}\{\mu(s,a)\}$. Moreover, $\E_{{\mu}}[\bMKQ] = D_{\mu} \MKQ$. 
 In terms of \(\bMKQ\) and \(\Lam\) operators \(F\) and \(\bar F\) can be written as 
 \begin{align}
 \label{eq:F_matrix_repr}
F(\theta, z) &= \Phi^\top \bigl(\Lam(z) r + \gamma \bMKQ(z) V_{\theta}^{\lambda} - \Lam(z) Q_{\theta}\bigr)\eqsp,\quad
\bar{F}(\theta) = \Phi^\top D_{\mu}\bigl(r + \gamma P V_{\theta}^{\lambda} - Q_{\theta}\bigr)\eqsp.
\end{align}
Then \eqref{eq:update_rule} can be rewritten as
\begin{align}
\label{eq:delta_clt_decomposition}
    \Delta_{t+1} &= \Delta_t + \alpha_t F(\theta_t, z_t)\\
    &= \Delta_t + \alpha_t(\bar F(\theta_t) - \bar F(\theta^\star)) + \alpha_t(F(\theta_t, z_t) - \bar F(\theta_t))\\
    &=  \Delta_t + \alpha_t(\bar F(\theta_t) - \bar F(\theta^\star)) + \alpha_t(F(\theta^\star, z_t) + \{F(\theta_t, z_t) -F(\theta^\star, z_t)\} + \{ \bar F(\theta^\star) - \bar F(\theta_t)\})\eqsp.
\end{align}
In the entropy-regularized setting, the operators \(\bar F\) and \(F\) are differentiable for all \(\theta\in\rset^d\), and their Jacobians admit the representations  
\begin{align}
\label{eq:jacobian_formula}
    \nabla F(\theta, z_t) &= - \Phi^\top \Lam_t (I - \gamma \bMKQ_t \Pi^{\lambda}_{\theta})\Phi\eqsp,\quad
    \nabla \bar F(\theta) =  - \Phi^\top D_{\mu} (I - \gamma \MKQ^{\pi^\lambda_{\theta}} )\Phi\eqsp,
\end{align}
where 
\(
\Pi^{\lambda}_{\theta}(s',(s,a)) = 1\{s=s'\}\pi^{\lambda}_{\theta}\,.
\) The identities in \eqref{eq:jacobian_formula} follow from a straightforward computation.
Note that we still have
\(
\PE_{z \sim \mu}[\nabla F(\theta,z)] = \nabla \bar F(\theta).
\)
The essential property is the Lipschitz continuity of the Jacobian in the spectral norm; see \Cref{lem:jacobian_smoothness}. This yields estimates of the nonlinear terms in the Taylor decompositions of the operators \(F\) and \(\bar F\):
\begin{align}
\label{eq:extracting_linear}
F(\theta_t,z_t) - F(\theta^\star,z_t)
&=
\nabla F(\theta^\star, z_t) \Delta_t
+ \wR_t(\Delta_t), \\
\bar F(\theta_t) - \bar F(\theta^\star)
&=
\nabla \bar F(\theta^\star) \Delta_t
+ \mathrm{R}(\Delta_t)\eqsp,
\end{align}
where the remainder terms satisfy
\[
\|\wR_t(\Delta_t)\|
\leq
\frac{\gamma}{2\lambda}\|\Delta_t\|^2,
\qquad
\|\mathrm{R}(\Delta_t)\|
\leq
\frac{\gamma}{2\lambda}\|\Delta_t\|^2\eqsp.
\]
(see \Cref{lem:jacobian_reminders} for details).
Extracting the linear term in \eqref{eq:delta_clt_decomposition} via \eqref{eq:extracting_linear}, we obtain
\begin{align}
\label{eq:delta_clt_decomposition_2}
  \Delta_{t+1}
    &=
    \bigl(
    I-\alpha_t\Phi^\top D_\mu
    (I-\gamma\MKQ^{\pi_{\theta^\star}^\lambda})
    \Phi
    \bigr)\Delta_t
    + 
    \alpha_t \wR_t(\Delta_t) \\
    &+ \alpha_t F(\theta^\star, z_t) +\alpha_t(\nabla F(\theta^\star, z_t) - \nabla \bar F(\theta^\star))(\theta_t - \theta^\star)\eqsp.
\end{align}
For simplicity, we denote the Bellman noise at optimality and the Jacobian fluctuation by 
\begin{align}
\label{eq:eps_and_E}
   \beps_t = F(\theta^\star, z_t)\eqsp, \quad  E_t = \nabla F(\theta^\star, z_t) - \nabla \bar F(\theta^\star)\eqsp.
\end{align}
The next step is to decompose the noise $\beps_t + E_t\Delta_t$ into two components via the Poisson equation \eqref{eq:pois_solution_equation}. Formally, let \(
    \beps_t = \beps_t^{(0)} + \beps_t^{(1)}
\) , and \(E_t = E_t^{(0)} + E_t^{(1)}\), where 
\begin{align}
\label{eq:beps_E_decompose}
&\begin{cases}
\beps_{t}^{(0)} = \pois{\varepsilon}{t} - \totMKQ \pois{\varepsilon}{t-1}\eqsp,\\[0.75em]
\beps_{t}^{(1)} = \totMKQ \pois{\varepsilon}{t-1} - \totMKQ \pois{\varepsilon}{t}\eqsp,
\end{cases}
&
&\begin{cases}
E_{t}^{(0)} = \pois{E}{t} - \totMKQ \pois{E}{t-1}\eqsp,\\[0.75em]
E_{t}^{(1)} = \totMKQ \pois{E}{t-1} - \totMKQ \pois{E}{t}\eqsp.
\end{cases}
\end{align}
The first component \( \beps_{t}^{(0)} + E_{t}^{(0)}\Delta_t \) forms a martingale difference sequence, whereas the second component \( \beps_{t}^{(1)} + E_{t}^{(1)}\Delta_t \) is a remainder term. Thus, we rewrite \eqref{eq:delta_clt_decomposition_2} as
\begin{align}
\label{eq:delta_clt_decomposition_3}
\Delta_{t+1}
&=
\bigl(
I-\alpha_t\Phi^\top D_\mu
(I-\gamma\MKQ^{\pi_{\theta^\star}^\lambda})
\Phi
\bigr)\Delta_t
+ \alpha_t \wR_t(\Delta_t) \\
&\quad
+ \alpha_t \bigl( \beps_t^{(0)} + E_t^{(0)}\Delta_t \bigr)
+ \alpha_t \bigl( \beps_t^{(1)} + E_t^{(1)}\Delta_t \bigr)\eqsp.
\end{align}
To unroll the recursion \eqref{eq:delta_clt_decomposition_3}, it is convenient to introduce the product of deterministic matrices
\begin{equation}\label{eq:Gamma_products}
\Gamma_{m:n}
:=
\prod_{j=m}^{n}
\bigl(
I - \alpha_j \Phi^\top D_\mu (I - \gamma \MKQ^{\pi^\lambda_{\theta^\star}}) \Phi
\bigr)\eqsp,
\end{equation}
with the convention that empty products are equal to \(I\). Unrolling \eqref{eq:delta_clt_decomposition_3} up to \(t=0\) we get 
\begin{align}
\label{eq:Delta_with_remainder}
    \Delta_{t+1} = \sum_{j=0}^{t}\alpha_j\Gamma_{j+1:t}  \beps_j^{(0)} + \Rlast_{t+1}\eqsp,
\end{align}
where the explicit expression for \(\Rlast_{t+1}\), together with the bound \(\PE^{1/p}[\|\Rlast_t\|^p]\lesssim_{pr}p^2\alpha_t\) is provided in \Cref{lem:Delta_with_remainder}.
\paragraph{Gaussian approximation.}
We measure the quality of approximation in terms of
\begin{equation}
\label{eq:berry-esseen}
\dconv (X,Y) = \sup_{A\in \mathcal{C}(\rset^{d})} | \Pb(X\in A) - \Pb(Y\in A)|,
\end{equation}
where $\mathcal{C}(\rset^{d})$ is the class of convex sets in $\rset^d$. Other choices of the class of sets over which one takes supremum in \eqref{eq:berry-esseen} are possible. A popular choice is the class of hyper-rectangles \cite{ChernozhukovChetverikovKato2013Gaussian,rubtsov2026gaussian}. While natural for algorithms without function approximation, this choice is problematic for $Q$-learning with function approximation. Indeed, particular coordinates are not necessarily interpretable in the transformed space, even for linear approximation. At the same time, convex distance is directly applicable for statistical procedures relying on bootstrap \cite{spokoiny2015,JMLR:v19:17-370}. That is why we state our results in convex distance \eqref{eq:berry-esseen}.
\par 
Under assumption \Cref{assum:UGE}, we define the Bellman noise covariance matrix \(\Sigmabf_{\beps}\) as
\begin{align}
\Sigmabf_{\beps} = \PE_{\mu}[\beps^{(0)}(\beps^{(0)})^\top]\eqsp.
 \end{align}
 Below we present the main result of this section, that is, that \(\sqrt{n}\bar{\Delta}_n\) converges in distribution to the Gaussian law with the covariance matrix $\Sigmabf_{\infty}$ given by
 \begin{align}
 \label{eq:G_definition}
\Sigmabf_{\infty} = G^{-1}\Sigmabf_{\beps}G^{-\top}\eqsp, \quad G = -\nabla \bar F(\theta^\star)\eqsp.
 \end{align}
\begin{theorem}
\label{thm:clt}
   Assume that \Cref{assum:regularity}--\Cref{assum:features} hold and let $Z \sim \mathcal{N}(0,I)$. Then, for
any \(n\geq 1\),  it holds that  
\begin{align}
    \dconv (\sqrt{n}\bar\Delta_n, \Sigmabf_{\infty}^{1/2} Z)&\lesssim_{pr}\frac{d^{1/2}\log^{1/4}(d)\log (n) }{n^{1/4}} + \frac{d^{1/4}\log^2(d n) }{n^{\omega-1/2}} + \frac{1}{n^{1-\omega}}\eqsp.
\end{align}
\end{theorem}
\begin{proof}
The proof begins by isolating the leading linear term. Substituting the representation \eqref{eq:Delta_with_remainder} into the averaged error \(\bar{\Delta}_n = \frac{1}{n}\sum_{t=1}^{n} \Delta_t\), we get that \(\sqrt{n}\bar{\Delta}_n  = W_n + \mathrm{R}_n^{\operatorname{pr}}\), with
\begin{align}
\label{eq: W_n definition}
    W_n = \frac{1}{\sqrt{n}}\sum_{t=0}^{n-1}\sum_{j=0}^{t}\alpha_j\Gamma_{j+1:t}\beps_j^{(0)}\eqsp, \quad \mathrm{R}_n^{\operatorname{pr}} =  \frac{1}{\sqrt{n}}\sum_{t=1}^n\mathrm{R}_t^{\operatorname{last}}\eqsp.
\end{align}
For the first term we change order of summation
\begin{align}\label{eq:Q_j_definition}
    W_n = \frac{1}{\sqrt{n}}\sum_{t=0}^{n-1} \sum_{j=0}^{t}\alpha_j\Gamma_{j+1:t}\beps_j^{(0)}   =  \frac{1}{\sqrt{n}}\sum_{j=0}^{n-1} \mathrm{Q}_{j}\beps_j^{(0)}  \, ,
\end{align}
where $\mathrm{Q}_j = \sum_{t=j}^{n-1} \alpha_j\Gamma_{j+1:t}$. Here, \(W_n\) is a weighted sum of martingale difference vectors with covariance matrix
\begin{align}
\label{eq:Sigma_n_def}
    \Sigmabf_{n}= \E[W_nW_n^\top] =\frac{1}{n}\sum_{t=0}^{n-1} Q_t\Sigmabf_{\beps}Q_t^\top\eqsp.
\end{align}
 Applying \Cref{lem:shao} with $X = W_n$  and $Y =\mathrm{R}_n^{\operatorname{pr}}$ yields
\begin{align}
\label{eq:shao_bound}
\dconv(\sqrt{n}\bar\Delta_n, \Sigmabf_{\infty}^{1/2} Z) \leq \dconv(W_n, \Sigmabf_\infty^{\frac{1}{2}} Z ) +  2 c_d^{\frac{p}{p+1}}
\mathbb{E}^{\frac{1}{p+1}}\!\left[\|\Sigmabf_\infty^{-\frac{1}{2}}\mathrm{R}_n^{\operatorname{pr}}\|^p\right] \eqsp,
\end{align}
where we used the invariance of the convex distance under invertible linear transformations.
Setting \(p=\log(dn)\) and using \Cref{lem:Delta_with_remainder_pr} we bound the second term in \eqref{eq:shao_bound} as follows 
\begin{align}\label{eq:shao_remainder}
&2 c_d^{\frac{p}{p+1}}\PE^{\frac{1}{p+1}}[\|\Sigmabf_\infty^{-\frac{1}{2}}\mathrm{R}_n^{\operatorname{pr}}\|^p]
\lesssim   \frac{p^2 c_d\operatorname{C}^{\operatorname{last}}\|\Sigmabf_\infty^{-\frac{1}{2}}\|}{(1-\omega)n^{\omega-1/2}}\lesssim \frac{d^{1/4}\log^2(dn)}{n^{\omega - 1/2}} \eqsp.
\end{align}
It is important to highlight that the moment mismatch
\(\PE^{\frac{1}{p+1}}[\|\cdot\|^p]\)
introduces only an absolute constant factor due to the choice
\(p=\log(dn)\).
For the second term in \eqref{eq:shao_bound} we apply the triangle inequality
\begin{align}
\dconv (W_n, \Sigmabf_{\infty}^{\frac{1}{2}} Z ) \leq \dconv (W_n, \Sigmabf_n^{\frac{1}{2}} Z ) + \dconv (\Sigmabf_n^{\frac{1}{2}} Z, \Sigmabf_{\infty}^{\frac{1}{2}} Z^\prime)\eqsp,
\end{align}
where $Z^\prime$ is an i.i.d. copy of $Z$. Next, CLT for martingales \cite[Corollary 4]{wu2025uncertainty} is applied
\begin{align}\label{eq:clt_distance}
    \dconv (W_n,\,\Sigmabf_n^{1/2}Z) 
    \lesssim
C_W\frac{\sqrt{d}\log n }{n^{1/4}} \eqsp,
\end{align}
where \(\operatorname{C}_W\) is defined in \Cref{lem:clt:application}.
Finally, by \Cref{lem:convex_dist} we get \begin{align}
\label{eq:gaussian_comparison}
\dconv (\Sigmabf_n^{\frac{1}{2}} Z, \Sigmabf_{\infty}^{\frac{1}{2}} Z^\prime)\le C_\infty n^{\omega - 1}\eqsp.
\end{align} 
The statement of the theorem now follows by combining the bounds in
\eqref{eq:shao_remainder}, \eqref{eq:clt_distance} and \eqref{eq:gaussian_comparison}.
\end{proof}
\paragraph{Discussion.} 
In classical Q-learning without regularization, the statement of \Cref{thm:clt} cannot be guaranteed without additional assumptions, such as MDP \(\mathcal{M}\) admitting a unique optimal policy \( \pi^\star \) with a strictly positive optimality gap
\begin{equation}
\kappa_{\mathcal{M}} := 
        \min_{s \in \cS}\,
        \min_{a \ne \pi^\star(s)}
        \bigl|V^\star(s) - Q^\star(s,a)\bigr|
        > 0 \eqsp,
\end{equation}
or other conditions (e.g. \cite{li2023statistical}[Lemma B.1]), enforcing the uniqueness of the optimal policy. Without it, the limiting covariance may fail to be uniquely defined, leading to ambiguity in the asymptotic behavior. In addition, as \(\gamma \to 1\), the optimality gap \(\kappa_{\mathcal{M}} \to 0\), which significantly slows down the convergence in the CLT. This uniqueness is unrealistic and is a limitation of the classical setting. Contrary, the entropy-regularized setting guarantees uniqueness of the  solution under \Cref{assum:features}.
\par 
Optimizing the bound of \Cref{thm:clt} over $\omega \in (1/2, 1)$ we get setting $\omega = 3/4$ that
\begin{equation}
\label{eq:conv_dist_q_learn}
\dconv(\sqrt{n}\bar\Delta_n, \Sigmabf_{\infty}^{\frac{1}{2}} Z)\lesssim_{pr}\frac{d^{1/2}\log^2(dn)}{n^{1/4}}\eqsp.
\end{equation}
The closest counterparts to \eqref{eq:conv_dist_q_learn} are the recent results of \cite{liu2025central} and \cite{rubtsov2026gaussian}. However, both papers study the setting of the classical tabular $Q$-learning without entropy regularization and function approximation. In particular, \cite[Corollary~5]{liu2025central} provides rate $n^{-1/6}$ in the $1$-st order Wasserstein distance, with additional polynomial dependence on the state-action space cardinality $|S| |A|$. With the standard relation between the convex and Wasserstein-1 distance \cite{nourdin2022multivariate}, this result implies $n^{-1/12}$ rate in convex distance, still with polynomial dependence in $|S| |A|$. The latter issue is resolved in \cite{rubtsov2026gaussian}, where the authors obtain the bound of order $O\br{\log(|\cS||\cA|n)n^{-1\slash 6}}$ for the convergence rate \eqref{eq:berry-esseen} measured for the class of hyper-rectangles. Our rate of \eqref{eq:conv_dist_q_learn} is sharper in terms of its dependence on $n$ and resolves the limitation of the uniqueness of optimal policy due to the structure of the entropy-regularized $Q$-learning itself.

\section{Conclusion}
\label{sec:conclusion}
In this paper, we derive a novel Gaussian approximation bound for entropy-regularized $Q$-learning in terms of convex distance with a rate $n^{-1/4}$. We expect that this result is an important prerequisite for non-asymptotic analysis of the statistical inference procedures based on entropy-regularized $Q$-learning. 

\newpage 
\bibliographystyle{plain}
\bibliography{ref}

@article{spokoiny2015,
	author = {Vladimir Spokoiny and Mayya Zhilova},
	doi = {10.1214/15-AOS1355},
	journal = {The Annals of Statistics},
	number = {6},
	pages = {2653 -- 2675},
	publisher = {Institute of Mathematical Statistics},
	title = {{Bootstrap confidence sets under model misspecification}},
	url = {https://doi.org/10.1214/15-AOS1355},
	volume = {43},
	year = {2015}
}

@article{shao2022berry,
  title={Berry--{E}sseen bounds for multivariate nonlinear statistics with applications to {M}-estimators and stochastic gradient descent algorithms},
  author={Shao, Qi-Man and Zhang, Zhuo-Song},
  journal={Bernoulli},
  volume={28},
  number={3},
  pages={1548--1576},
  year={2022},
  publisher={Bernoulli Society for Mathematical Statistics and Probability}
}

@article{mnih2015human,
  title={Human-level control through deep reinforcement learning},
  author={Mnih, Volodymyr and Kavukcuoglu, Koray and Silver, David and Rusu, Andrei A and Veness, Joel and Bellemare, Marc G and Graves, Alex and Riedmiller, Martin and Fidjeland, Andreas K and Ostrovski, Georg and others},
  journal={nature},
  volume={518},
  number={7540},
  pages={529--533},
  year={2015},
  publisher={Nature Publishing Group}
}

@article{nourdin2022multivariate,
  title={Multivariate normal approximation on the Wiener space: new bounds in the convex distance},
  author={Nourdin, Ivan and Peccati, Giovanni and Yang, Xiaochuan},
  journal={Journal of Theoretical Probability},
  volume={35},
  number={3},
  pages={2020--2037},
  year={2022},
  publisher={Springer}
}

@inproceedings{kaledin2020finite,
	author = {Kaledin, M. and Moulines, E. and Naumov, A. and Tadic, V. and Wai, Hoi-To},
	booktitle = {Conference On Learning Theory},
	title = {Finite time analysis of linear two-timescale stochastic approximation with Markovian noise},
	year = {2020}
}

@book{douc:moulines:priouret:soulier:2018,
    AUTHOR = {Douc, Randal and Moulines, Eric and Priouret, Pierre and Soulier, Philippe},
     TITLE = {Markov chains},
    SERIES = {Springer Series in Operations Research and Financial Engineering},
 PUBLISHER = {Springer},
      YEAR = {2018},
     PAGES = {xviii+757},
      ISBN = {978-3-319-97703-4}
}

@inproceedings{qu2020finite,
  author       = {Qu, Guannan and Wierman, Adam},
  title        = {Finite-time analysis of asynchronous stochastic approximation and Q-learning},
  booktitle    = {Conference on Learning Theory (COLT)},
  series       = {Proceedings of Machine Learning Research},
  volume       = {125},
  pages        = {3185--3205},
  publisher    = {PMLR},
  year         = {2020}
}

@article{watkins1992qlearning,
  author       = {Watkins, Christopher J. C. H. and Dayan, Peter},
  title        = {Q-learning},
  journal      = {Machine Learning},
  volume       = {8},
  pages        = {279--292},
  year         = {1992},
  doi          = {10.1007/BF00992698}
}

@inproceedings{li2023statistical,
  title={A statistical analysis of polyak-ruppert averaged q-learning},
  author={Li, Xiang and Yang, Wenhao and Liang, Jiadong and Zhang, Zhihua and Jordan, Michael I},
  booktitle={International Conference on Artificial Intelligence and Statistics},
  pages={2207--2261},
  year={2023},
  organization={PMLR}
}

@article{li2024minimax,
  author       = {Li, Gen and Wu, Weichen and Chi, Yuejie and Ma, Cong and Rinaldo, Alessandro and Wei, Yuting},
  title        = {Is Q-learning minimax optimal? A tight sample complexity analysis},
  journal      = {Operations Research},
  volume       = {72},
  number       = {1},
  pages        = {222--236},
  year         = {2024},
  doi          = {10.1287/opre.2023.2490}
}

@article{liu2025central,
  title={Central Limit Theorems for Asynchronous Averaged Q-Learning},
  author={Liu, Xingtu},
  journal={arXiv preprint arXiv:2509.18964},
  year={2025}
}

@book{sutton2018reinforcement,
  author       = {Sutton, Richard S. and Barto, Andrew G.},
  title        = {Reinforcement Learning: An Introduction},
  publisher    = {MIT Press},
  year         = {2018},
  edition      = {2nd}
}

@book{bertsekas1996neuro,
  author       = {Bertsekas, Dimitri P. and Tsitsiklis, John N.},
  title        = {Neuro-Dynamic Programming},
  publisher    = {Athena Scientific},
  year         = {1996}
}

@article{wainwright2019variance,
  title={Variance-reduced $ Q $-learning is minimax optimal},
  author={Wainwright, Martin J},
  journal={arXiv preprint arXiv:1906.04697},
  year={2019}
}

@techreport{ruppert1988efficient,
  title={Efficient estimations from a slowly convergent {R}obbins-{M}onro process},
  author={Ruppert, David},
  year={1988},
  institution={Cornell University Operations Research and Industrial Engineering}
}

@article{polyak1992acceleration,
  title={Acceleration of stochastic approximation by averaging},
  author={Polyak, Boris T and Juditsky, Anatoli B},
  journal={SIAM journal on control and optimization},
  volume={30},
  number={4},
  pages={838--855},
  year={1992},
  publisher={SIAM}
}

@article{wu2025uncertainty,
  title={Uncertainty quantification for {M}arkov chains with application to temporal difference learning},
  author={Wu, Weichen and Wei, Yuting and Rinaldo, Alessandro},
  journal={arXiv preprint arXiv:2502.13822},
  year={2025}
}

@article{ChernozhukovChetverikovKato2013Gaussian,
  author    = {Chernozhukov, Victor and Chetverikov, Denis and Kato, Kengo},
  title     = {Gaussian approximations and multiplier bootstrap for maxima of sums of high-dimensional random vectors},
  journal   = {The Annals of Statistics},
  year      = {2013},
  volume    = {41},
  number    = {6},
  pages     = {2786--2819},
  doi       = {10.1214/13-AOS1161},
  url       = {https://doi.org/10.1214/13-AOS1161},
  publisher = {Institute of Mathematical Statistics}
}

@article{jin2018q,
  title={Is Q-learning provably efficient?},
  author={Jin, Chi and Allen-Zhu, Zeyuan and Bubeck, Sebastien and Jordan, Michael I},
  journal={Advances in neural information processing systems},
  volume={31},
  year={2018}
}

@article{sheshukova2025gaussian,
  title={Gaussian approximation and multiplier bootstrap for stochastic gradient descent},
  author={Sheshukova, Marina and Samsonov, Sergey and Belomestny, Denis and Moulines, Eric and Shao, Qi-Man and Zhang, Zhuo-Song and Naumov, Alexey},
  journal={arXiv preprint arXiv:2502.06719},
  year={2025}
}

@article{samsonov2024gaussian,
  title={Gaussian approximation and multiplier bootstrap for polyak-ruppert averaged linear stochastic approximation with applications to td learning},
  author={Samsonov, Sergey and Moulines, Eric and Shao, Qi-Man and Zhang, Zhuo-Song and Naumov, Alexey},
  journal={Advances in Neural Information Processing Systems},
  volume={37},
  pages={12408--12460},
  year={2024}
}

@inproceedings{butyrin2026gaussian,
  title={Gaussian approximation for two-timescale linear stochastic approximation},
  author={Butyrin, Bogdan and Rubtsov, Artemy and Naumov, Alexey and Ulyanov, Vladimir V and Samsonov, Sergey},
  booktitle={Proceedings of the AAAI Conference on Artificial Intelligence},
  volume={40},
  number={43},
  pages={36627--36635},
  year={2026}
}

@article{wu2024statistical,
  title={{S}tatistical {I}nference for {T}emporal {D}ifference {L}earning with {L}inear {F}unction {A}pproximation},
  author={Wu, Weichen and Li, Gen and Wei, Yuting and Rinaldo, Alessandro},
  journal={arXiv preprint arXiv:2410.16106},
  year={2024}
}

@article{JMLR:v19:17-370,
  author  = {Yixin Fang and Jinfeng Xu and Lei Yang},
  title   = {Online Bootstrap Confidence Intervals for the Stochastic Gradient Descent Estimator},
  journal = {Journal of Machine Learning Research},
  year    = {2018},
  volume  = {19},
  number  = {78},
  pages   = {1--21},
  url     = {http://jmlr.org/papers/v19/17-370.html}
}

@article{rubtsov2026gaussian,
  title={Gaussian Approximation for Asynchronous Q-learning},
  author={Rubtsov, Artemy and Samsonov, Sergey and Ulyanov, Vladimir and Naumov, Alexey},
  journal={arXiv preprint arXiv:2604.07323},
  year={2026}
}

@inproceedings{melo2008analysis,
  title={An analysis of reinforcement learning with function approximation},
  author={Melo, Francisco S and Meyn, Sean P and Ribeiro, M Isabel},
  booktitle={Proceedings of the 25th international conference on Machine learning},
  pages={664--671},
  year={2008}
}

@article{zou2019finite,
  title={Finite-sample analysis for sarsa with linear function approximation},
  author={Zou, Shaofeng and Xu, Tengyu and Liang, Yingbin},
  journal={Advances in neural information processing systems},
  volume={32},
  year={2019}
}

@article{xie2022statistical,
  title={A statistical online inference approach in averaged stochastic approximation},
  author={Xie, Chuhan and Zhang, Zhihua},
  journal={Advances in Neural Information Processing Systems},
  volume={35},
  pages={8998--9009},
  year={2022}
}

@inproceedings{samsonov2026statistical,
title={Statistical inference for Linear Stochastic Approximation with {M}arkovian Noise},
author={Sergey Samsonov and Marina Sheshukova and Eric Moulines and Alexey Naumov},
booktitle={The Thirty-ninth Annual Conference on Neural Information Processing Systems},
year={2026},
url={https://openreview.net/forum?id=nWTQREGMLG}
}

@article{tsitsiklis1994asynchronous,
  title={Asynchronous stochastic approximation and Q-learning},
  author={Tsitsiklis, John N},
  journal={Machine learning},
  volume={16},
  number={3},
  pages={185--202},
  year={1994},
  publisher={Springer}
}

@article{chen2022finite,
  title={Finite-sample analysis of nonlinear stochastic approximation with applications in reinforcement learning},
  author={Chen, Zaiwei and Zhang, Sheng and Doan, Thinh T and Clarke, John-Paul and Maguluri, Siva Theja},
  journal={Automatica},
  volume={146},
  pages={110623},
  year={2022},
  publisher={Elsevier}
}

@inproceedings{haarnoja2017reinforcement,
  title={Reinforcement learning with deep energy-based policies},
  author={Haarnoja, Tuomas and Tang, Haoran and Abbeel, Pieter and Levine, Sergey},
  booktitle={International conference on machine learning},
  pages={1352--1361},
  year={2017},
  organization={PMLR}
}

@inproceedings{haarnoja2018soft,
  title={Soft actor-critic: Off-policy maximum entropy deep reinforcement learning with a stochastic actor},
  author={Haarnoja, Tuomas and Zhou, Aurick and Abbeel, Pieter and Levine, Sergey},
  booktitle={International conference on machine learning},
  pages={1861--1870},
  year={2018},
  organization={Pmlr}
}

@inproceedings{ziebart2008maximum,
  title={Maximum entropy inverse reinforcement learning.},
  author={Ziebart, Brian D and Maas, Andrew L and Bagnell, J Andrew and Dey, Anind K and others},
  booktitle={Aaai},
  volume={8},
  pages={1433--1438},
  year={2008},
  organization={Chicago, IL, USA}
}

@book{osekowski,
   title =     {Sharp Martingale and Semimartingale Inequalities},
   author =    {Osekowski, A.},
   publisher = {Birkhäuser Basel},
   isbn =      {3034803699,9783034803694},
   year =      {2012},
   series =    {Monografie Matematyczne 72},
   edition =   {1},
   volume =    {}
}

@incollection{baird1995residual,
  title={Residual algorithms: Reinforcement learning with function approximation},
  author={Baird, Leemon},
  booktitle={Machine learning proceedings 1995},
  pages={30--37},
  year={1995},
  publisher={Elsevier}
}

@inproceedings{durmus2021tight,
 author = {Durmus, Alain and Moulines, Eric and Naumov, Alexey and Samsonov, Sergey and Scaman, Kevin and Wai, Hoi-To},
 booktitle = {Advances in Neural Information Processing Systems},
 pages = {30063--30074},
 publisher = {Curran Associates, Inc.},
 title = {Tight High Probability Bounds for Linear Stochastic Approximation with Fixed Stepsize},
 url = {https://proceedings.neurips.cc/paper/2021/file/fc95fa5740ba01a870cfa52f671fe1e4-Paper.pdf},
 volume = {34},
 year = {2021}
}

@article{durmus2022finite,
author = {Durmus, Alain and Moulines, Eric and Naumov, Alexey and Samsonov, Sergey},
title = {Finite-Time High-Probability Bounds for {P}olyak–{R}uppert Averaged Iterates of Linear Stochastic Approximation},
journal = {Mathematics of Operations Research},
volume = {50},
number = {2},
pages = {935-964},
year = {2025},
doi = {10.1287/moor.2022.0179},
URL = {https://doi.org/10.1287/moor.2022.0179},
eprint = {https://doi.org/10.1287/moor.2022.0179}
}

@article{ball1993reverse,
  title={The reverse isoperimetric problem for Gaussian measure},
  author={Ball, Keith},
  journal={Discrete \& Computational Geometry},
  volume={10},
  number={4},
  pages={411--420},
  year={1993},
  publisher={Springer-Verlag Berlin, Heidelberg}
}

@article{paulin2015concentration,
  title={Concentration inequalities for Markov chains by Marton couplings and spectral methods},
  author={Paulin, Daniel},
  journal={Electronic journal of probability},
  volume={20},
  pages={79},
  year={2015},
  publisher={Institute of Mathematical Statistics}
}

\newpage
\appendix

\section{Notations}
\begin{table}[htbp]
\centering
\caption{Notation used throughout the paper}
\label{tab:notation}
\begin{tabular}{ll}
\toprule
\textbf{Notation} & \textbf{Meaning} \\
\midrule
\multicolumn{2}{l}{\textbf{Markov Decision Process}} \\
\midrule
\( \mathcal{M} \) & Markov Decision Process (MDP) \\
\( \cS \) & State space, \( |\cS| = S \) \\
\( \cA \) & Action space, \( |\cA| = A \) \\
\( r \) & Reward function \\
\( \MKQ\) & Transition kernel of the MDP \( \mathcal{M} \) \\
\(\gamma\) & Discount factor\\
\(\lambda\) & Entropy regularization parameter\\
\(  \Phi\) & Feature matrix\\

\midrule
\multicolumn{2}{l}{\textbf{Policies and Value Functions}} \\
\midrule
\( \pi_b \) & Behaviour policy \\
\( \pi^{\lambda}_{\theta} \) & Entropy-regularized softmax policy parameterized by \(\theta\) \\
\( Q_{\theta} \) & Action-value function parameterized by \(\theta\) \\
\( V^{\lambda}_{\theta}\) & Entropy-regularized value function parameterized by \(\theta\) \\

\midrule
\multicolumn{2}{l}{\textbf{Markov Chains and Kernels}} \\
\midrule
\(\totMKQ\)  & Transition kernel of the Markov chain \(z_t=(s_t,a_t,s_{t+1})\) induced by \(\pi_b\) \\
\( \mu \) & Stationary distribution of \( \totMKQ \) \\
\( \taumix \) & Mixing time of \( \totMKQ \) \\

\midrule
\multicolumn{2}{l}{\textbf{Learning Process}} \\
\midrule
\( \alpha_k \) & Step–size sequence \\
\( \theta_t \) & Q–learning iterates \\
\(\theta^\star\) & Fixed point of the projected Bellman operator\\

\midrule
\multicolumn{2}{l}{\textbf{General Notation}} \\
\midrule
\( \mathbf{1}\{A\} \) & Indicator function of event \(A\) \\
\( \|{\cdot}\| \) & Spectral norm (vector/matrix) \\
\( \Pb, \PE \) & Probability and expectation \\
\( \operatorname{diag}(X) \) & Diagonal vector of matrix \(X\) \\
\( \e_s \) & One–hot vector corresponding to state \(s\) \\
\(\pois{f}{}\) & Poisson transformation of function \(f\)\\
\(\dconv\) & Convex distance\\
\bottomrule
\end{tabular}
\end{table}

\section{Proof of \Cref{thm:moments}}
\label{sec:proof_of_moments}
In \Cref{lem:F_lipshitz} and \Cref{lem:strong_convexity} we summarize the main properties of the operator \(F\).
\begin{lemma}
\label{lem:F_lipshitz} 
Assume that \Cref{assum:regularity}--\Cref{assum:features} hold. Then
the operator \(F\) is \((1+\gamma)\)-Lipschitz in \(\theta\), i.e., for all \(z\) and all \(\theta_1, \theta_2 \in \rset^d\),
\begin{align}
\|F(\theta_1,z)-F(\theta_2,z)\|
\le
(1+\gamma)\|\theta_1-\theta_2\| \eqsp.
\end{align}
\end{lemma}
\begin{proof}
Let \(z=(s,a,s')\). By definition,
\begin{align}
F(\theta_1,z)-F(\theta_2,z)
&=
\phi(s,a)
\Big(
\gamma\big(V_{\theta_1}^{\lambda}(s')-V_{\theta_2}^{\lambda}(s')\big)
-
\big(Q_{\theta_1}(s,a)-Q_{\theta_2}(s,a)\big)
\Big) \\
&=
\phi(s,a)
\Big(
\gamma\big(V_{\theta_1}^{\lambda}(s')-V_{\theta_2}^{\lambda}(s')\big)
-
\phi(s,a)^\top (\theta_1 - \theta_2)
\Big).
\end{align}
Using the triangle inequality, the Lipschitz property of the soft value map, and \(\|\phi(s,a)\|\le 1\), we obtain
\begin{align}
\|F(\theta_1,z)-F(\theta_2,z)\|
&\le
\|\phi(s,a)\|
\left(
\gamma
\big|V_{\theta_1}^{\lambda}(s')-V_{\theta_2}^{\lambda}(s')\big|
+
\big|\phi(s,a)^\top(\theta_1-\theta_2)\big|
\right) \\
&\le
\gamma \max_{a'\in\cA}
\big|\phi(s',a')^\top(\theta_1-\theta_2)\big|
+
|\phi(s,a)^\top(\theta_1-\theta_2)| \\
&\le
(1+\gamma)\|(\theta_1-\theta_2)\| .
\end{align}
This proves the claim.
\end{proof}
\begin{lemma}
\label{lem:strong_convexity} 
Assume that \Cref{assum:regularity}--\Cref{assum:features} hold. Then, for all \(\theta \in \rset^d\),\[\big\langle \theta - \theta^\star,\bar F(\theta)-\bar F(\theta^\star)\big\rangle\le-\frac{\kappa}{2}\|\theta-\theta^\star\|^2 \eqsp.\]
\end{lemma}
\begin{proof} Let \(\Delta=\theta-\theta^\star\). Since \(Q_\theta-Q_{\theta^\star}=\Phi\Delta\), we have\[\bar F(\theta)-\bar F(\theta^\star)=\Phi^\top D_\mu\left(\gamma \MKQ (V_\theta^\lambda-V_{\theta^\star}^\lambda)-\Phi\Delta\right).\]Therefore,\begin{align}\big\langle \Delta,\bar F(\theta)-\bar F(\theta^\star)\big\rangle&=\gamma\PE_{\mu}\left[\phi(s,a)^\top\Delta\,\bigl(V_\theta^\lambda(s')-V_{\theta^\star}^\lambda(s')\bigr)\right] \\&\quad-\PE_{\mu}\left[\bigl(\phi(s,a)^\top\Delta\bigr)^2\right].\end{align}Using \(\gamma ab-a^2\le \frac{\gamma^2}{2}b^2-\frac12 a^2\), we obtain\[\big\langle \Delta,\bar F(\theta)-\bar F(\theta^\star)\big\rangle\le\frac12\left[\gamma^2\PE_{\mu}\left[\bigl(V_\theta^\lambda(s')-V_{\theta^\star}^\lambda(s')\bigr)^2\right]-\PE_{\mu}\left[\bigl(\phi(s,a)^\top\Delta\bigr)^2\right]\right].\]By the mean-value theorem and the identity \(\nabla_\theta V_\vartheta^\lambda(s)=\PE_{a\sim \pi_\vartheta^\lambda(\cdot\mid s)}[\phi(s,a)]\), for every \(s\) there exists an integral representation\[V_\theta^\lambda(s)-V_{\theta^\star}^\lambda(s)=\int_0^1\left\langle\PE_{a\sim \pi_{\theta^\star+u\Delta}^\lambda(\cdot\mid s)}[\phi(s,a)],\Delta\right\rangle du .\]Hence, by Jensen's inequality,\begin{align}\bigl(V_\theta^\lambda(s)-V_{\theta^\star}^\lambda(s)\bigr)^2&\le\int_0^1\PE_{a\sim \pi_{\theta^\star+u\Delta}^\lambda(\cdot\mid s)}\left[\bigl(\phi(s,a)^\top\Delta\bigr)^2\right]du .\end{align}Averaging over \(s'\sim \mu\) gives
\[
\PE_{\mu}\left[\bigl(V_\theta^\lambda(s')-V_{\theta^\star}^\lambda(s')\bigr)^2\right]\le\int_0^1\Delta^\top\Sigma_{\phi}^{{\theta^\star+u\Delta}}\Delta\,du .
\]
Moreover,\[\PE_{\mu}\left[\bigl(\phi(s,a)^\top\Delta\bigr)^2\right]=\Delta^\top \Sigma_{\phi}^{\pi_b}\Delta .\]Combining the above bounds and using \Cref{assum:features}  for each \(\theta^\star+u\Delta\), we get
\begin{align}\big\langle \Delta,\bar F(\theta)-\bar F(\theta^\star)\big\rangle&\le\frac12\int_0^1\Delta^\top\left(\gamma^2\Sigma_{\phi}^{{\theta^\star+u\Delta}}-\Sigma_{\phi}^{\pi_b}\right)\Delta\,du \le-\frac{\kappa}{2}\|\Delta\|^2 .\end{align}This proves the claim.\end{proof}


 In \Cref{lem:xi_t_bound}, we control the noise \(\xi_t\) in terms of a universal constant
and the deviation from \(\theta^\star\). We note that the possible growth of the
noise term with respect to \(\theta_t\) is standard  in
general SA schemes. 
\begin{lemma}
\label{lem:xi_t_bound}
Assume that \Cref{assum:regularity}--\Cref{assum:features} hold. Then,
the noise terms \(F(\theta^\star, z_t)\) and \(\xi_t\) satisfy the bounds
\begin{align}
\|F(\theta^\star, z_t)\|\leq \eps_{\max}\eqsp, \quad\|\xi_t\|^2
\leq \Cxi(1+ \|\theta_t - \theta^\star\|^2) \eqsp,
\end{align}
where
\begin{align}
\eps_{\max} = 1 + (1+\gamma)\|\theta^\star\|_2 + \gamma\lambda \log(A)\eqsp,\quad
\Cxi = 2\eps_{\max}^2
\vee 8(1+\gamma)^2 .
\end{align}
\end{lemma}

\begin{proof}
First, we bound the noise at the optimum:
\begin{align}
\|F(\theta^\star, z_t)\|
&=
\left\|
\Phi^\top\bigl(
\Lam_t r
+
\gamma \bMKQ_t V_{\theta^\star}^{\lambda}
-
\Lam_t Q_{\theta^\star}
\bigr)
\right\| \\
&\leq
\|\phi(s_t,a_t) r(s_t,a_t)\|
+ \gamma \|\phi(s_t,a_t) V_{\theta^\star}^{\lambda}(s_{t+1})\|
+ \|\phi(s_t,a_t) Q_{\theta^\star}(s_t,a_t)\| \\
&\leq
1 + (1+\gamma)\|Q_{\theta^\star}\|_{\infty}
+ \gamma\lambda \log(A) \\
&\leq
1 + (1+\gamma)\|\theta^\star\|_2
+ \gamma\lambda \log(A) := \eps_{\max}\eqsp.
\end{align}
By \Cref{lem:F_lipshitz}, we also have
\begin{align}
&\left\|
\bigl\{F(\theta_t, z_t) - F(\theta^\star, z_t)\bigr\}
+
\bigl\{\bar{F}(\theta^\star) - \bar{F}(\theta_t)\bigr\}
\right\| \leq
(2+2\gamma)\|\theta_t - \theta^\star\| \eqsp.
\end{align}
Therefore, using \(\|a+b\|^2 \leq 2\|a\|^2 + 2\|b\|^2\), we obtain the result bound.
\end{proof}
Next, \Cref{lem:poisson_transform_uniform_bound} provides the main technical results associated with the Poisson equation framework.
Recall that the Poisson equation 
\begin{equation}
\label{eq:pois_solution_equation}
\pois{f}{} -  \bar{\MKQ} \pois{f}{} = f - \E_{\mu}[f]\eqsp, 
\end{equation}
under Assumption \Cref{assum:UGE} has a unique solution for any bounded measurable $f$, which is given by the formula 
\begin{align}
\label{eq:poisson_solution}
\pois{f}{} = \sum_{k=0}^{\infty}\left(\bar\MKQ^{k}f - \E_{\mu}[f]\right)\eqsp.
\end{align}

Moreover, using  \Cref{assum:UGE}, one can show that $\pois{f}{}$ is also bounded.
\begin{lemma}
\label{lem:poisson_transform_uniform_bound}
Suppose that \Cref{assum:UGE} holds. Let \(f:\mathcal Z\to\mathbb R^d\) be such that
\(
\mathbb E_{\mu}[f(z)] = 0\) and 
\(\|f(z)\|_2 \le a\) for all \( z\in\mathcal Z 
\).
Define the Poisson transform of \(f\) by \eqref{eq:poisson_solution}.
Then \(\pois{f}{}(z)\) is well-defined and satisfies the uniform bound
\[
\sup_{z\in\mathcal Z}\|g^f(z)\|_2
\le
\frac{8}{3}a\taumix .
\]
\end{lemma}
\begin{proof}
Since \(\mathbb E_{\mu}[f(Z)]=0\), for every \(z\in\mathcal Z\) and \(t\ge 0\),
\[
\totMKQ^{t}f(z)
=
\int_{\mathcal Z} f(z')
\big(
\bar{\MKQ}^{t}(dz'\mid z)-\mu(dz')
\big).
\]
Hence, using the total variation measure and \Cref{assum:UGE},
\[
\left\|
\bar{\MKQ}^{t}f(z)
\right\|
\le
2a\,
d_{\operatorname{tv}}
\bigl(
\bar{\MKQ}^{t}(\cdot\mid z),\mu \bigr)\le 
2a
\left(1/4\right)^{\lfloor t/\taumix\rfloor}
\eqsp.
\]
Therefore,
\begin{align}
\|g^f(z)\|
&\le
\sum_{t=0}^{\infty}
\big\|
\bar{\MKQ}^{t}f(z)
\big\| \le
2a
\sum_{k=0}^{\infty}
\left(\frac14\right)^{\lfloor k/\taumix\rfloor}\le
\frac{8}{3}a\,\taumix \eqsp.
\end{align}
Taking the supremum over \(z\in\mathcal Z\) completes the proof.
\end{proof}
Since the noise term \(\xi_t\) (see \eqref{eq:xi_t_definition}) depends not only on \(z_t\), but also on the deviation \(\theta_t - \theta^\star\), the Poisson decomposition \eqref{eq:poisson_solution} cannot be applied directly. To simplify the derivations, it is convenient to use the matrix notation for operators and variables introduced in \Cref{subsec:gaus_approx}. We rewrite \(\xi_t\) using the matrix operators \(\bMKQ_t\) and \(\Lam_t\), introduced in \eqref{eq:P_and_Lam}:
\begin{align}
\xi_{t}
&= \Phi^\top(\Lam_t r+\gamma \bMKQ_{t}V_{\theta^\star}^{\lambda} - \Lam_t Q_{\theta^\star}) +
\gamma \Phi^\top(\bMKQ_{t} - D_\mu \MKQ)(V_{\theta_t}^{\lambda} - V_{\theta^\star}^{\lambda})
-
\Phi^\top(\Lam_t - D_\mu)(Q_{\theta_t}-Q_{\theta^\star}) \eqsp.
\end{align}
The representation above follows directly from the matrix forms of \(F\) and \(\bar F\) in \eqref{eq:F_matrix_repr}. Although \(\xi_t\) is not uniformly bounded, the corresponding martingale and remainder components can still be expressed through the Poisson transforms \(\pois{\bMKQ}{}\) and \(\pois{\Lam}{}\), which are uniformly bounded by \Cref{lem:poisson_transform_uniform_bound}, together with the algorithmic errors \(Q_{\theta_t}-Q_{\theta^\star}\) and \(V_{\theta_t}^{\lambda} - V_{\theta^\star}\). Specifically, let
\(
\xi_t = \xi_t^{(0)} + \xi_t^{(1)}
\),
where
    \begin{align}
    \label{eq:xi_1_2_definitions}
    \begin{cases}
        \xi_{t}^{(0)} = (\pois{\varepsilon}{t} - \totMKQ \pois{\varepsilon}{t-1}) + \gamma\Phi^\top(\pois{\bMKQ}{t} - \totMKQ \pois{\bMKQ}{t-1})(V_{\theta_t}^{\lambda}-V_{\theta^\star}^{\lambda}) - \Phi^\top(\pois{\Lam}{t} - \totMKQ \pois{\Lam}{t-1})(Q_{\theta_t}-Q_{\theta^\star})\eqsp,\\[0.75em]
        \xi_{t}^{(1)} = (\totMKQ \pois{\varepsilon}{t-1} - \totMKQ \pois{\varepsilon}{t}) +\gamma\Phi^\top(\totMKQ \pois{\bMKQ}{t-1} - \totMKQ \pois{\bMKQ}{t})(V_{\theta_t}^{\lambda}- V_{\theta^\star}^{\lambda}) - \Phi^\top(\totMKQ \pois{\Lam}{t-1} - \totMKQ \pois{\Lam}{t})(Q_{\theta_t} - Q_{\theta^\star})\eqsp,
    \end{cases}
\end{align}
and \(\beps_t = F(\theta^\star, z_t)\) is defined as in \eqref{eq:eps_and_E}. To verify that \eqref{eq:xi_1_2_definitions} indeed yields
\(\xi_t = \xi_t^{(0)} + \xi_t^{(1)}\), it suffices to use the identities
\(\PE_{\mu}\big[\Lam\big] = D_{\mu}\),
\(\PE_{\mu}\big[\bMKQ\big] = D_{\mu}\MKQ\), and
\(\PE_{\mu}\big[\beps\big] = 0\).


Next, we prove \Cref{lem:square_error_rec}, an important intermediate step in
the proof of \Cref{thm:moments}. This lemma bounds \(\|\Delta_{t+1}\|^2\) in
terms of the cross terms \(\langle \Delta_j,\xi_j^{(0)}\rangle\) and
\(\langle \Delta_j,\xi_j^{(1)}\rangle\).
\paragraph{Proof of \Cref{lem:square_error_rec}}.
\begin{proof}
We first decompose \(\xi_t = \xi_t^{(0)} + \xi_t^{(1)}\) and then expand the squared error:
    \begin{align}
\label{eq:delta_first_step}
\|\Delta_{t+1}\|^2  &= \big\langle  \Delta_t + \alpha_t F(\theta_t, z_t), \Delta_t + \alpha_t F(\theta_t, z_t) \big\rangle\\
        &=\|\Delta_t\|^2 + 2\alpha_t \langle \Delta_t,\bar F(\theta_t) + \xi_t \rangle + \alpha_t^2 \|F(\theta_t, z_t)\|^2\\
        &= \|\Delta_t\|^2 + 2\alpha_t\langle\Delta_t, \bar F(\theta_t)\rangle + 2\alpha_t \langle \Delta_t, \xi_t^{(0)} \rangle + 2\alpha_t \langle \Delta_t, \xi_t^{(1)} \rangle+ \alpha_t^2 \|F(\theta_t, z_t)\|^2 \eqsp. 
\end{align}
By \Cref{lem:strong_convexity}, we have
\begin{align}
    \langle\Delta_t, \bar F(\theta_t)\rangle\leq -\frac{\kappa}{2}\|\Delta_t\|^2\eqsp.
    \end{align}
Moreover, by \Cref{lem:F_lipshitz} and \Cref{lem:xi_t_bound},
\begin{align}
     \quad \|F(\theta_t, z_t)\|^2 \leq 2\|\bar F(\theta_t) - \bar F(\theta^\star)\|^2 + 2\|\xi_t\|^2 \leq (2+2\gamma)\|\Delta_t\|^2 + 2\Cxi(1 + \|\Delta_t\|^2)\eqsp.
\end{align}
Substituting these estimates into \eqref{eq:delta_first_step}, we obtain
\begin{align}
 \|\Delta_{t+1}\|^2&\leq (1 - \kappa\alpha_t + 2\alpha_t^2(\Cxi+1+\gamma))\|\Delta_t\|^2 + 2\alpha_t \langle \Delta_t, \xi_t^{(0)} \rangle + 2\alpha_t \langle \Delta_t, \xi_t^{(1)} \rangle  + 2\alpha_t^2 \Cxi\\
        &\leq  (1 - \kappa\alpha_t/2)\|\Delta_t\|^2 + 2\alpha_t \langle \Delta_t, \xi_t^{(0)} \rangle + 2\alpha_t \langle \Delta_t, \xi_t^{(1)} \rangle  + 2\alpha_t^2\Cxi\eqsp,
    \end{align}
where the last step uses \Cref{assum:steps}.
Unrolling the recursion up to \(t=0\), we get
\begin{align}
    \|\Delta_{t+1}\|^2 &\leq \gprod{0}{t}\|\Delta_0\|^2 + 2\sum_{j=0}^t\alpha_j\gprod{j+1}{t}\big\langle \Delta_j, \xi_j^{(0)}\rangle + 2\sum_{j=0}^t\alpha_j\gprod{j+1}{t}\big\langle \Delta_j, \xi_j^{(1)}\rangle + 2\Cxi\sum_{j=0}^t\alpha_j^2\gprod{j+1}{t}\\
    &\leq \gprod{0}{t}\|\Delta_0\|^2 + 2\sum_{j=0}^t\alpha_j\gprod{j+1}{t}\big\langle \Delta_j, \xi_j^{(0)}\rangle + 2\sum_{j=0}^t\alpha_j\gprod{j+1}{t}\big\langle \Delta_j, \xi_j^{(1)}\rangle + \frac{8\Cxi\alpha_t}{\kappa}\eqsp,
\end{align}
where the last inequality follows from
\Cref{lem:rate_of_convergence} and the
definition of \(\gprod{m}{n}\) in \eqref{eq:gprod_definition}. Finally, we use \Cref{lem:P_alpha_ineq} and obtain \(\alpha_0\gprod{1}{t}\leq \alpha_t\), and hence \(\gprod{0}{t}\|\Delta_0\|^2\leq \alpha_t\alpha_0^{-1}\|\Delta_0\|^2\). Substituting into the above bound yields the statement of the lemma.
\end{proof}
Next, \Cref{lem:xi_0_bound} and \Cref{lem:xi_1_bound} allow us to close the
estimate in \Cref{lem:square_error_rec} in terms of \(\moment{j}{p}\).
\begin{lemma}
\label{lem:xi_0_bound}
Assume that \Cref{assum:regularity}--\Cref{assum:features} hold. Then, for all \(t\geq 0\)
    \begin{align}
\PE^{2/p}\Big[\Big\|\sum_{j=0}^t\alpha_j\gprod{j+1}{t}\big\langle\Delta_j, \xi_j^{(0)}\rangle\Big\|^{p}\Big] \lesssim  p^2\taumix^2\Cxi\sum_{j=0}^t \alpha_j^2\gprod{j+1}{t}^2\left(\sqrt{\moment{j}{p}} + \moment{j}{p}\right)\eqsp.
    \end{align}
\end{lemma}
\begin{proof}
 Denote by \( \mathcal{F}_j = \sigma( z_0, \ldots, z_j) \)  
the sigma-algebra generated by all observations up to time \( j \). By construction, \(\{\xi_t^{(0)}\}_{t\geq0}\) forms a martingale difference sequence with respect to \(\{\mathcal{F}_t\}_{t\geq0}\). Indeed, \(\xi_t^{(0)}\), defined in \eqref{eq:xi_1_2_definitions}, has the form
\begin{align}
    \xi_{t}^{(0)} = (\pois{\varepsilon}{t} - \totMKQ \pois{\varepsilon}{t-1}) + \gamma\Phi^\top(\pois{\bMKQ}{t} - \totMKQ \pois{\bMKQ}{t-1})(V_{\theta_t}^{\lambda}-V_{\theta^\star}^{\lambda}) - \Phi^\top(\pois{\Lam}{t} - \totMKQ \pois{\Lam}{t-1})(Q_{\theta_t}-Q_{\theta^\star})\eqsp.
\end{align}
The approximation errors \(Q_{\theta_t} - Q_{\theta^\star}\) and \(V_{\theta_t}^{\lambda} - V_{\theta^\star}^{\lambda}\) are \(\mathcal{F}_{t-1}\)-measurable, while the prefactors form martingale differences induced by the Markov chain:
\begin{align}
    \PE [\pois{\varepsilon}{t} \mid \mathcal{F}_{t-1}] = \totMKQ \pois{\varepsilon}{t-1}\eqsp, \quad 
    \PE [\pois{\bMKQ}{t} \mid \mathcal{F}_{t-1}] = \totMKQ \pois{\bMKQ}{t-1}\eqsp,
    \quad
    \PE [\pois{\Lam}{t} \mid \mathcal{F}_{t-1}] = \totMKQ \pois{\Lam}{t-1}\eqsp.
\end{align}
Hence,
\(
\PE[\xi_t^{(0)}\mid\mathcal{F}_{t-1}] =0
\). Burkholder’s inequality \cite[Theorem 8.1]{osekowski} gives
    \begin{align} &\PE^{2/p}\Big[\Big\|\sum_{j=0}^t\alpha_j\gprod{j+1}{t}\big\langle \Delta_j, \xi_j^{(0)}\rangle\Big\|^{p}\Big]  \\
    &\quad\lesssim p^2\sum_{j=0}^t \PE^{2/p}[\|\alpha_j\gprod{j+1}{t}\langle \Delta_j, \xi_j^{(0)}\rangle\|^{p}]\\
        &\quad\lesssim p^2\sum_{j=0}^t           \alpha_j^2\gprod{j+1}{t}^2\PE^{2/p}[\|\Delta_j\|^p\|\xi_j^{(0)}\|^p]\label{eq:mideq}\eqsp.
    \end{align}
Thus, it remains to bound \(\|\xi_j^{(0)}\|\). By \Cref{lem:xi_t_bound} and \Cref{lem:poisson_transform_uniform_bound}, we obtain
\begin{align}
    \|(\pois{\varepsilon}{j} - \totMKQ \pois{\varepsilon}{j-1})\|
    \lesssim
    \taumix\eps_{\max}\eqsp.
\end{align}
Further, \Cref{lem:poisson_transform_uniform_bound} again yields
\begin{align}
\|\Phi^\top(\pois{\Lam}{j} - \totMKQ \pois{\Lam}{j-1})(Q_{\theta_j}-Q_{\theta^\star})\|
&\lesssim
\|\Phi^\top(\pois{\Lam}{j} - \totMKQ \pois{\Lam}{j-1})\Phi\|
\|\Delta_j\|\lesssim
\taumix \|\Delta_j\|\eqsp,
\end{align}
where we used the linearity of the Poisson transform under multiplication by deterministic matrices together with the bound
\(
\|\Phi^\top \Lam \Phi\|\le 1
\). Finally, observe that
\begin{align}
\big\|\Phi^\top P_j (V_{\theta_j}^{\lambda}-V_{\theta^\star}^{\lambda})\big\|_2
&=
\big\|\Phi^\top e_{s_j,a_j} e_{s_{j+1}}^\top
(V_{\theta_j}^{\lambda}-V_{\theta^\star}^{\lambda})\big\|_2 \\
&=
\|\phi(s_j,a_j)\|_2
\big|V_{\theta_j}^{\lambda}(s_{j+1})
-
V_{\theta^\star}^{\lambda}(s_{j+1})\big| \\
&\leq
\big|V_{\theta_j}^{\lambda}(s_{j+1})
-
V_{\theta^\star}^{\lambda}(s_{j+1})\big| \\
&\leq
\sup_{a\in\cA}
\big|\phi(s_{j+1},a)^\top(\theta_j-\theta^\star)\big| \\
&\leq
\|\theta_j-\theta^\star\|_2\eqsp.
\end{align}
Together with the linearity of the Poisson transform and \Cref{lem:poisson_transform_uniform_bound}, this implies
\begin{align}
\|\Phi^\top(\pois{\bMKQ}{j} - \totMKQ \pois{\bMKQ}{j-1})(V_{\theta_j}^{\lambda}-V_{\theta^\star}^{\lambda})\|
\lesssim
\taumix\|\Delta_j\|\eqsp.
\end{align}
Hence, we obtain the uniform bound
\begin{align}
\label{eq:xi_j^0_bound}
    \|\xi_j^{(0)}\|\lesssim \taumix (\eps_{\max} + \|\Delta_j\|).
\end{align}
Using \eqref{eq:xi_j^0_bound} together with \(\eps_{\max}\le\sqrt{\Cxi}\), we obtain \(\|\xi_j^{(0)}\|\|\Delta_j\|\lesssim \taumix\sqrt{\Cxi}(\|\Delta_j\| + \|\Delta_j\|^2)\).
Consequently,
\begin{align}
\label{eq:lyapunov_application}
    \PE^{2/p}[\|\Delta_j\|^p\|\xi_j^{(0)}\|^p]&\lesssim \Cxi\taumix^2\PE^{2/p}[(\|\Delta_j\| + \|\Delta_j\|^2)^p]\\
    &\lesssim  2\Cxi\taumix^2\big(\PE^{2/p}[\|\Delta_j\|^p] +\PE^{2/p}[\|\Delta_j\|^{2p}]\big)\\
    &\overset{(a)}{\lesssim} \Cxi\taumix^2\PE^{1/p}[\|\Delta_j\|^{2p}] + \Cxi\taumix^2\PE^{2/p}[\|\Delta_j\|^{2p}]\\
    &= \Cxi \taumix^2\big(\sqrt{\moment{j}{p}} + \moment{j}{p}\big)\eqsp,
\end{align}
where \((a)\) follows from the Lyapunov inequality.
Thus we can continue \eqref{eq:mideq} as 
\begin{align}
\PE^{2/p}\Big[\Big\|\sum_{j=0}^t\alpha_j\gprod{j+1}{t}\big\langle \Delta_j, \xi_j^{(0)}\rangle\Big\|^{p}\Big]  \lesssim p^2\Cxi\taumix^2\sum_{j=0}^t \alpha_j^2\gprod{j+1}{t}^2\left(\sqrt{\moment{j}{p}} + \moment{j}{p}\right)\eqsp.
\end{align}
\end{proof}
\Cref{lem:discret_ibp} provides a standard discrete integration-by-parts identity.
We state it for completeness.
\begin{lemma}
\label{lem:discret_ibp}
Let \((v_j)_{j=0}^t\) and \((A_j)_{j=0}^t\) be sequences of vectors, and set \(A_{-1}=0\). Then
\begin{align}
\sum_{j=0}^t
\big\langle v_j, A_{j-1} - A_j \big\rangle
&=
-\big\langle v_t, A_t \big\rangle
+
\sum_{j=0}^{t-1}
\big\langle v_{j+1} - v_j, A_j \big\rangle .
\end{align}
\end{lemma}
\begin{lemma}
\label{lem:xi_1_bound}
Assume that \Cref{assum:regularity}--\Cref{assum:features} hold. Then, for all \(t\geq 0\)
\begin{align}
\PE^{2/p}\Big[\Big\|\sum_{j=0}^t\alpha_j\gprod{j+1}{t}\big\langle \Delta_j, \xi_j^{(1)}\rangle\Big\|^{p}\Big] &\lesssim  \Com\taumix^2  \Cxi^2\sum_{j=0}^{t}\alpha_j^2\gprod{j+1}{t}^2 \Big( \sqrt{\moment{j}{p}} + {\moment{j}{p}}\Big) \\
&+ \frac{\Cxi^2\taumix^2 \alpha_t^2}{\kappa^2}\eqsp,
\end{align}
where \(\Com = \sum_{j=0}^{\infty}\alpha_j^2\).
\end{lemma}
\begin{proof}
We first note that under \Cref{assum:steps}, we have \(\omega\in(1/2,1)\). Hence,
\(\Com<\infty\), and this constant depends only on \(\omega\) and \(c_0\). For simplicity, introduce \(\pois{\xi}{t}\) defined by
\begin{align}
    \pois{\xi}{t} = \pois{\varepsilon}{t}  + \gamma\Phi^\top\pois{\bMKQ}{t}(V_{\theta_t}^{\lambda}-V_{\theta^\star}^{\lambda}) -\Phi^\top  \pois{\Lam}{t}(Q_{\theta_t}-Q_{\theta^\star})\eqsp,
\end{align}
By the same argument as in the proof of \Cref{lem:xi_0_bound}, which is used to control \(\|\xi_j^{(0)}\|\), we obtain
\begin{align}
\label{eq:Pgxi_bound}
    \|\pois{\xi}{j}\|\lesssim\taumix(\sqrt{\Cxi} + \|\Delta_j\|)\eqsp, \quad \|\totMKQ \pois{\xi}{j}\|\lesssim\taumix(\sqrt{\Cxi} + \|\Delta_j\|)\eqsp.
\end{align}
We first recall definition of \(\xi_t^{(1)}\) given in \eqref{eq:xi_1_2_definitions}:
\begin{align}
    \xi_{t}^{(1)} = (\totMKQ \pois{\varepsilon}{t-1} - \totMKQ \pois{\varepsilon}{t}) + \gamma\Phi^\top(\totMKQ \pois{\bMKQ}{t-1} - \totMKQ \pois{\bMKQ}{t})(V_{\theta_t}^{\lambda}- V_{\theta^\star}^{\lambda}) -\Phi^\top(\totMKQ \pois{\Lam}{t-1} - \totMKQ \pois{\Lam}{t})(Q_{\theta_t} - Q_{\theta^\star})\eqsp.
\end{align}
In order to directly apply \Cref{lem:discret_ibp}, we decompose \(\xi_t^{(1)}\) as follows:
\begin{align}
    \xi_t^{(1)} = \big(\totMKQ\pois{\xi}{t-1} - \totMKQ\pois{\xi}{t} \big) + \gamma \Phi^\top\totMKQ \pois{\bMKQ}{t-1}(V_{\theta_{t}}^{\lambda} - V_{\theta_{t-1}}^{\lambda}) - \Phi^\top\totMKQ\pois{\Lam}{t-1}(Q_{\theta_{t}} - Q_{\theta_{t-1}}) \eqsp.
\end{align}
By \Cref{lem:discret_ibp} with \(u_j = \alpha_j \gprod{j+1}{t}\Delta_j\) and \(
A_j = \totMKQ \pois{\xi}{j}\),
\begin{align}
 \sum_{j=0}^t
\alpha_j \gprod{j+1}{t}
\big\langle \Delta_j, \totMKQ \pois{\xi}{j-1} - \totMKQ \pois{\xi}{j}\big\rangle &=
-\alpha_t
\big\langle \Delta_t, \totMKQ \pois{\xi}{t} \big\rangle +
\sum_{j=0}^{t-1}
\left(
\alpha_{j+1}\gprod{j+2}{t}
-
\alpha_j \gprod{j+1}{t}
\right)
\big\langle \Delta_j, \totMKQ \pois{\xi}{j} \big\rangle \\
&+
\sum_{j=0}^{t-1}
\alpha_{j+1}\gprod{j+2}{t}
\big\langle
\Delta_{j+1}-\Delta_j,
\totMKQ \pois{\xi}{j}
\big\rangle .
\end{align}
We now analyze the coefficient term:
\begin{align}
    |\alpha_{j+1}\gprod{j+2}{t} - \alpha_j\gprod{j+1}{t}|
    &= |\alpha_{j+1} - \alpha_j + \alpha_j\alpha_{j+1}\kappa/4|\gprod{j+2}{t}\lesssim \alpha_j^2\gprod{j+2}{t}\eqsp,
\end{align}
where the last inequality follows from \Cref{lem:alpha_delta}, which gives
\(|\alpha_j-\alpha_{j+1}|\lesssim \alpha_j^2\).
Then, by Minkowski's inequality
\begin{align}
\PE^{2/p}\Big[\|\sum_{j=0}^t
\alpha_j \gprod{j+1}{t}
\big\langle \Delta_j, \xi_j^{(1)} \big\rangle\|^p\Big] &\lesssim 
{\alpha_t^2
\PE^{2/p}[\|\big\langle \Delta_t, \totMKQ \pois{\xi}{t} \big\rangle \|^p]}\tag*{\(( T_1)\)}\\
&+
{\PE^{2/p}\Big[\|\sum_{j=0}^{t-1}
\alpha_j^2\gprod{j+1}{t}
\big\langle \Delta_j, \totMKQ \pois{\xi}{j} \big\rangle\|^p\Big] }\tag*{\((T_2)\)}\\
&+
{\PE^{2/p}\Big[\|\sum_{j=0}^{t-1}
\alpha_{j+1}\gprod{j+2}{t}
\big\langle
\Delta_{j+1}-\Delta_j,
\totMKQ \pois{\xi}{j}
\big\rangle\|^p\Big]} \tag*{\((T_3)\)}\\
&+ {\PE^{2/p}\Big[\|\sum_{j=0}^{t-1}
\alpha_{j}\gprod{j+1}{t}
\big\langle
\Delta_j,
\Phi^\top\totMKQ\pois{\Lam}{j-1}(Q_{\theta_{j}} - Q_{\theta_{j-1}})
\big\rangle\|^p\Big]} \tag*{\(( T_4)\)}\\
&+ {\PE^{2/p}\Big[\|\sum_{j=0}^{t-1}
\alpha_{j}\gprod{j+1}{t}
\big\langle
\Delta_j,
\gamma \Phi^\top\totMKQ \pois{\bMKQ}{j-1}(V_{\theta_{j}} - V_{\theta_{j-1}})
\big\rangle\|^p\Big]}\tag*{\((T_5)\)}\eqsp.
\end{align}
We now bound each term separately.
\paragraph{Bound for \(T_1\).} By \eqref{eq:Pgxi_bound}, we have
\(
\| \totMKQ \pois{\xi}{t} \|\|\Delta_t\|
\lesssim
\taumix\sqrt{\Cxi}(\|\Delta_t\| + \|\Delta_t\|^2)
\).
Using the same argument as in \eqref{eq:lyapunov_application}, we obtain
\begin{align}
    \alpha_t^2
\PE^{2/p}[\|\big\langle \Delta_t, \totMKQ \pois{\xi}{t} \big\rangle \|^p] &\lesssim \alpha_t^2\Cxi \taumix^2\big(\sqrt{\moment{t}{p}} + \moment{t}{p}\big)\eqsp.
\end{align}
The contribution of \(T_1\) is grouped with the summation term arising from \(T_2\) in the final bound.
\paragraph{Bound for \(T_2\).}
Observe that, by \eqref{eq:Pgxi_bound},
\[
\big|
\big\langle \Delta_j, \totMKQ \pois{\xi}{j} \big\rangle
\big|
\lesssim
\taumix \sqrt{\Cxi}(\|\Delta_j\| + \|\Delta_j\|^2)\eqsp.
\]
Substituting the above bound into \(T_2\) and applying Minkowski's inequality, we obtain
\begin{align}
    &\PE^{2/p}\Big[\|\sum_{j=0}^{t-1}
\alpha_j^2\gprod{j+1}{t}
\big\langle \Delta_j, \totMKQ \pois{\xi}{j} \big\rangle\|^p\Big] \lesssim \Cxi\taumix^2 \Big(\sum_{j=0}^{t-1}\alpha_j^2\gprod{j+1}{t}( \PE^{1/p}[\|\Delta_j\|^p] + \PE^{1/p}[\|\Delta_j\|^{2p}])\Big)^2\\
&\qquad \overset{(a)}{\lesssim}  \Cxi\taumix^2 \Big(\sum_{j=0}^{t-1}\alpha_j^2\gprod{j+1}{t}(  \PE^{1/(2p)}[\|\Delta_j\|^{2p}] + \PE^{1/p}[\|\Delta_j\|^{2p}])\Big)^2\\
&\qquad \overset{(b)}{\lesssim} \Cxi\taumix^2 \Big(\sum_{j=0}^t \alpha_j^2\Big)\Big(\sum_{j=0}^{t-1}\alpha_j^2\gprod{j+1}{t}^2 \PE^{1/p}[\|\Delta_j\|^{2p}] + \sum_{j=0}^{t-1}\alpha_j^2\gprod{j+1}{t}^2 \PE^{2/p}[\|\Delta_j\|^{2p}]\Big)\\
&\qquad \overset{(c)}{\lesssim} \Cxi\Com\taumix^2  \sum_{j=0}^{t-1}\alpha_j^2\gprod{j+1}{t}^2 \Big( \sqrt{\moment{j}{p}} + {\moment{j}{p}}\Big)\eqsp.
\end{align}
Here, \((a)\) follows from Lyapunov's inequality, which gives
\(
\PE^{1/p}[\|\Delta_j\|^p]
\le
\PE^{1/(2p)}[\|\Delta_j\|^{2p}]
\),
step \((b)\) follows from the Cauchy--Schwarz inequality. Finally, \((c)\)
uses the definitions of \(\Com\) and \(\moment{j}{p}\).
\paragraph{Bound for \(T_3\).}
Observe that, by \Cref{lem:F_lipshitz} and \Cref{lem:xi_t_bound}
\begin{align}
\label{eq:Delta_Delt}
    \|\Delta_{j+1} - \Delta_j\| &=\|\theta_{j+1} - \theta_j\| = \alpha_j\|F(\theta_j, z_j)\|\\
    &\leq \alpha_j\|F(\theta_j, z_j) - F(\theta^\star, z_j)\|  + \alpha_j\|F(\theta^\star, z_j)\|\\
    &\leq 3\alpha_j\|\Delta_j\| + \alpha_j\eps_{\max}\eqsp.
\end{align}
Hence,
\begin{align}
     \|\Delta_{j+1} - \Delta_j\|\totMKQ \pois{\xi}{j} &\|\leq \alpha_j\sqrt{\Cxi}\taumix(3\|\Delta_j\| + \eps_{\max})(1 + \|\Delta_j\|)\\
     &= \alpha_j\sqrt{\Cxi}\taumix(\eps_{\max} + (3+\eps_{\max})\|\Delta_j\| + 3\|\Delta_j\|^2)\\
     &\lesssim \alpha_j{\Cxi}\taumix (1+\|\Delta_j\| + \|\Delta_j\|^2)\eqsp.
\end{align}
Substituting the above bound into \(T_3\) and applying Minkowski's inequality we obtain
\begin{align}
    &\PE^{2/p}\Big[\|\sum_{j=0}^{t-1}
\alpha_{j+1}\gprod{j+2}{t}
\big\langle
\Delta_{j+1}-\Delta_j,
\totMKQ \pois{\xi}{j}
\big\rangle\|^p\Big]\\
&\qquad \lesssim \Cxi^2\taumix^2 \Big(\sum_{j=0}^{t-1}\alpha_j^2\gprod{j+1}{t}\Big(1 + \PE^{1/p}[\|\Delta_j\|^p] + \PE^{1/p}[\|\Delta_j\|^{2p}]\Big)\Big)^2\\
&\qquad \overset{(a)}{\lesssim} \frac{\Cxi^2\taumix^2 \alpha_t^2}{\kappa^2} + \Cxi^2\taumix^2 \Big(\sum_{j=0}^{t-1}\alpha_j^2\gprod{j+1}{t}(  \PE^{1/(2p)}[\|\Delta_j\|^{2p}] + \PE^{1/p}[\|\Delta_j\|^{2p}])\Big)^2\\
&\qquad\lesssim\frac{\Cxi^2\taumix^2 \alpha_t^2}{\kappa^2} +  \Cxi^2\taumix^2 \Big(\sum_{j=0}^t \alpha_j^2\Big)\sum_{j=0}^{t-1}\alpha_j^2\gprod{j+1}{t}^2 \Big(\PE^{1/p}[\|\Delta_j\|^{2p}] +  \PE^{2/p}[\|\Delta_j\|^{2p}]\Big)\\
&\qquad\lesssim \frac{\Cxi^2\taumix^2 \alpha_t^2}{\kappa^2}  + \Com\Cxi^2\taumix^2  \sum_{j=0}^{t-1}\alpha_j^2\gprod{j+1}{t}^2 \Big( \sqrt{\moment{j}{p}} + {\moment{j}{p}}\Big)\eqsp.
\end{align}
where \((a)\) follows from \Cref{lem:rate_of_convergence} and Lyapunov’s inequality.
\paragraph{Bound for \(T_4\).}
Using \eqref{eq:Delta_Delt}, we obtain
\begin{align}
\label{eq:no_words}
\big\langle
\Delta_j,
\Phi^\top\totMKQ\pois{\Lam}{j-1}(Q_{\theta_{j}} - Q_{\theta_{j-1}})
\big\rangle
&=
\big\langle
\Delta_j,
\Phi^\top\totMKQ\pois{\Lam}{j-1}\Phi({\theta_{j}} - {\theta_{j-1}})
\big\rangle\\
&\leq
\|\Delta_j\|
\|\Phi^\top\totMKQ\pois{\Lam}{j-1}\Phi\|
\|{\theta_{j}} - {\theta_{j-1}}\|\\
&\lesssim
\taumix \|\Delta_j\|
(\alpha_j\|\Delta_{j-1}\| + \alpha_j\eps_{\max})\\
&\lesssim
\alpha_j\|\Delta_j\|^2
+
\alpha_j\|\Delta_{j-1}\|^2
+
\alpha_j\eps_{\max}\|\Delta_j\|\eqsp.
\end{align}
Then, following the same lines as in the bound for \(T_3\), we obtain
\begin{align}
\label{eq:bound_bound}
{\PE^{2/p}\Big[\big\|\sum_{j=0}^{t-1}
\alpha_{j}\gprod{j+1}{t}
\big\langle
\Delta_j,
\Phi^\top\totMKQ\pois{\Lam}{j-1}(Q_{\theta_{j}} - Q_{\theta_{j-1}})
\big\rangle\big\|^p\Big]}
\lesssim
\Com\Cxi^2\taumix^2
\sum_{j=0}^{t-1}
\alpha_j^2\gprod{j+1}{t}^2
\Big(
\sqrt{\moment{j}{p}}
+
{\moment{j}{p}}
\Big).
\end{align}
We note that although \eqref{eq:no_words} involves both
\(\|\Delta_j\|^2\) and \(\|\Delta_{j-1}\|^2\), after a suitable reindexing and using the equivalence
\(
\alpha_j\gprod{j+1}{t}
\asymp
\alpha_{j+1}\gprod{j+2}{t},
\)
we still arrive at \eqref{eq:bound_bound}.
\paragraph{Bound for \(T_5\).}
Repeating the same argument as in the bound for \(T_4\), we obtain
\begin{align}
\label{eq:bound_bound_T5}
{\PE^{2/p}\Big[\|\sum_{j=0}^{t-1}
\alpha_{j}\gprod{j+1}{t}
\big\langle
\Delta_j,
\gamma \Phi^\top\totMKQ \pois{\bMKQ}{j-1}(V_{\theta_{j}}^{\lambda} - V_{\theta_{j-1}}^{\lambda})
\big\rangle\|^p\Big]}
\lesssim
\Com\Cxi^2\taumix^2
\sum_{j=0}^{t-1}
\alpha_j^2\gprod{j+1}{t}^2
\Big(
\sqrt{\moment{j}{p}}
+
{\moment{j}{p}}
\Big)\eqsp.
\end{align}
Collecting the estimates for \(T_1,\ldots,T_5\) yields the claimed bound and completes the proof.
\end{proof}
We now restate \Cref{thm:moments} with an explicit value of \(\Cmoment\) and
provide the full proof.
\begin{lemma}
\label{lem:union_recursion}
   Assume that \Cref{assum:regularity}--\Cref{assum:features} hold. Then, for any $t>0$ and $p \ge 2$, it holds that
\[\moment{t+1}{p} \le \Cmoment p^4\alpha_t^2\eqsp,\]
where constants \(C_0,C_1\) and \(C_2\) are defined in \eqref{eq:C_i_definitions} and 
\begin{align}
   \Cmoment \asymp\frac{\|\Delta_0\|^4}{p^4\alpha_0^2}\vee C_1\vee \frac{C_2^2}{\kappa^2}\eqsp.
\end{align}

\end{lemma}
\begin{proof}
    Applying Minkowski's inequality to the result of \Cref{lem:square_error_rec}, together with the standard inequality \((a+b+c)^2\leq 3(a^2 + b^2 + c^2)\), we obtain
\begin{align}
    \moment{t+1}{p} &\leq 6\alpha_t^2\Big( \frac{64\Cxi^2}{\kappa^2} + \frac{\|\Delta_0\|^4}{\alpha_0^2}\Big) + 12\PE^{2/p}\Big[\Big\|\sum_{j=0}^t\alpha_j\gprod{j+1}{t}\big\langle \Delta_j, \xi_j^{(0)}\rangle\Big\|\Big] 
    \\
&+12\PE^{2/p}\Big[\Big\|\sum_{j=0}^t\alpha_j\gprod{j+1}{t}\big\langle \Delta_j, \xi_j^{(1)}\rangle \Big\|\Big]\eqsp.
\end{align}
Substituting results from \Cref{lem:xi_0_bound} and \Cref{lem:xi_1_bound} into the above bound leads to the following recursion inequality
\begin{align}
    \moment{t+1}{p} \leq C_0(\alpha_t/\alpha_0)^2 +   C_1\alpha_t^2 + p^2C_2\sum_{j=0}^{t}\alpha_j^2\gprod{j+1}{t}^2 \Big( \sqrt{\moment{j}{p}} + {\moment{j}{p}}\Big)\eqsp,
\end{align}
where
\begin{align}
\label{eq:C_i_definitions}
    C_0 &\asymp \|\Delta_0\|^4\eqsp,\quad
     C_1 \asymp  \frac{\Cxi^2\taumix^2 }{\kappa^2}\eqsp,\quad
    C_2 \asymp \taumix^2\Cxi^2  \Com  \eqsp.
\end{align}
We prove the claim by induction. Assume that 
\begin{align}
\label{eq:induction_assumption}
\moment{j}{p}\leq \Cmoment p^4 \alpha_j^2
\end{align}
holds for all \(j \le t\). We now verify the bound for \(j=t+1\). Indeed,
\begin{align}
\label{eq:indiction_reducing}
    \moment{t+1}{p} &\leq C_0(\alpha_t/\alpha_0)^2 +   C_1\alpha_t^2 + p^2C_2\sum_{j=0}^{t}\alpha_j^2\gprod{j+1}{t}^2 \Big( \sqrt{\moment{j}{p}} + {\moment{j}{p}}\Big)\\
    &\overset{(a)}{\le} C_0(\alpha_t/\alpha_0)^2 +  C_1\alpha_t^2 + p^4C_2\sqrt{\Cmoment}\sum_{j=0}^{t}\alpha_j^3\gprod{j+1}{t}^2  +  p^6C_2\Cmoment\sum_{j=0}^{t}\alpha_j^4\gprod{j+1}{t}^2 \\
    &\overset{(b)}{\leq} C_0(\alpha_t/\alpha_0)^2 +   C_1\alpha_t^2 +\frac{ 16 p^4 C_2\sqrt{\Cmoment}\alpha_t^2}{\kappa}  +  \frac{16 p^6C_2\Cmoment\alpha_t^3}{\kappa}\eqsp.
\end{align}
where step \((a)\) uses induction assumption \eqref{eq:induction_assumption}, step \((b)\) uses \Cref{lem:rate_of_convergence}.To complete the induction step, we bound each term in
\eqref{eq:indiction_reducing} by  \(\frac{p^4\Cmoment}{4}\). This is guaranteed under the
following conditions:
\begin{align}
    \begin{cases}
        \frac{C_0}{\alpha_0^2} \leq \frac{p^4\Cmoment}{4} \Leftarrow \frac{ \|\Delta_0\|^4}{p^4\alpha_0^2}\lesssim \Cmoment\eqsp,
        \\[0.75em]
        C_1\leq \frac{p^4\Cmoment}{4}\Leftarrow C_1\lesssim \Cmoment\eqsp,\\[0.75em]
        \frac{16 p^4 C_2\sqrt{\Cmoment}}{\kappa} \leq \frac{p^4\Cmoment}{4} \Leftarrow \frac{C_2^2}{\kappa^2} \lesssim \Cmoment\eqsp,\\[0.75em]
        \frac{16p^6\alpha_tC_2\Cmoment}{\kappa} \leq \frac{p^4\Cmoment}{4} \Leftarrow \alpha_0 \lesssim \frac{\kappa}{p^2 C_2}\eqsp.
    \end{cases}
\end{align}
Thus, it is enough to choose 
\begin{align}
     \Cmoment \asymp\frac{\|\Delta_0\|^4}{p^4\alpha_0^2}\vee C_1\vee \frac{C_2^2}{\kappa^2} \eqsp, \quad  \alpha_0\leq \frac{\kappa}{ p^2C_2}\eqsp.
\end{align}
\end{proof}

\section{Bounds for the Polyak--Ruppert Remainder}
\begin{lemma}
\label{lem:jacobian_smoothness}
Assume that \Cref{assum:regularity}--\Cref{assum:features} hold. Then, 
the Jacobian \(\nabla F(\theta, z)\) is \((\gamma/\lambda)\)-Lipschitz in \(\theta\), i.e., for all \(z\in\mathcal{Z}\) and all \(\theta_1, \theta_2 \in \rset^d\),
\begin{align}
\|\nabla F(\theta_1, z) - \nabla F(\theta_2, z)\|
\leq \frac{\gamma}{\lambda}\|\theta_1 - \theta_2\| \eqsp.
\end{align}
\end{lemma}
\begin{proof}
Let \(z = (s,a,s')\). We first write the Jacobian difference using the explicit formula in \eqref{eq:jacobian_formula}:
    \begin{align}
         \|\nabla F (\theta_1, z) - \nabla F(\theta_2, z)\|
         &= \gamma\big\|\Phi^\top \Lam(z) \bMKQ(z)\big(\Pi^{\lambda}_{\theta_1} -   \Pi^{\lambda}_{\theta_2}\big)\Phi\big\|\\
         &\overset{(a)}{=}\gamma\big\|\sum_{a'\in\cA}\phi(s,a)\phi^{\top}(s',a') \big(\pi^{\lambda}_{\theta_1}(a'\mid  s')- \pi^{\lambda}_{\theta_2}(a'\mid  s')\big) \big\|\\
    &\overset{(b)}{\leq}\gamma\sum_{a'\in\cA}\big\|\phi(s,a)\phi^{\top}(s',a') \big\|\big|\pi^{\lambda}_{\theta_1}(a'\mid  s')- \pi^{\lambda}_{\theta_2}(a'\mid  s')\big|\\
    &\overset{(c)}{\leq }\gamma \|\pi^{\lambda}_{\theta_1}(\,\cdot\mid  s')- \pi^{\lambda}_{\theta_2}(\,\cdot\mid  s')\|_1\eqsp,
         \end{align}
where \((a)\) follows from a direct computation, \((b)\) from the triangle inequality for the spectral norm, and \((c)\) from \Cref{assum:features}.
The desired bound follows from the classical Lipschitz property of the softmax map with temperature \(\lambda\), together with \Cref{assum:features}:
    \begin{align}
          \gamma \|\pi^{\lambda}_{\theta_1}(\,\cdot\mid  s')- \pi^{\lambda}_{\theta_2}(\,\cdot\mid  s')\|_1 &\leq \frac{\gamma}{\lambda}\|\Phi\theta_1(s', \cdot) - \Phi\theta_2(s', \cdot)\|_{\infty}\\
          &=\frac{\gamma}{\lambda}\sup_{a'\in\cA}\big|\phi^\top(s',a')(\theta_1 - \theta_2)\big|\\
          &\leq\frac{\gamma}{\lambda}\|\theta_1 - \theta_2\|\eqsp.
    \end{align}
\end{proof}
\begin{lemma}
\label{lem:F_bar_smooth}
Assume that \Cref{assum:regularity}--\Cref{assum:features} hold. Then, 
the Jacobian \(\nabla \bar F(\theta)\) is \((\gamma/\lambda)\)-Lipschitz in \(\theta\), i.e., for all \(\theta_1, \theta_2 \in \rset^d\),
\begin{align}
\|\nabla \bar F(\theta_1) - \nabla \bar F(\theta_2)\|
\leq \frac{\gamma}{\lambda}\|\theta_1 - \theta_2\| \eqsp.
\end{align}
\end{lemma}
\begin{proof}
    The claim follows from \Cref{lem:jacobian_smoothness} by averaging over \(z\sim\mu\):
    \begin{align}
        \|\nabla \bar F(\theta_1) - \nabla \bar F(\theta_2)\| &= \big\|\PE_{\mu}\big[\nabla F(\theta_1, z) - \nabla F(\theta_2, z)\big]\big\|\\
        &\leq \PE_{\mu}\big[\big\|\nabla F(\theta_1, z) - \nabla F(\theta_2, z)\big\|\big]\\
        &\leq\frac{\gamma}{\lambda}\big\|\theta_1 - \theta_2\big\|\eqsp.
    \end{align}
\end{proof}
Jacobian smoothness in \Cref{lem:jacobian_smoothness} and \Cref{lem:F_bar_smooth} yields the following estimates for the nonlinear terms in the Taylor decompositions of \(F\) and \(\bar F\).

\begin{lemma}\label{lem:jacobian_reminders}
Assume that \Cref{assum:regularity}--\Cref{assum:features} hold. Then, for all \(t\geq 0\),
\begin{align}
F(\theta_t,z_t) -F(\theta^\star,z_t)
&=
- \Phi^\top \Lam_t (I - \gamma \bMKQ_t \Pi^{\lambda}_{\theta^\star})\Phi\Delta_t
+ \wR_t(\Delta_t),\\
\bar{F}(\theta_t) - \bar F(\theta^\star)
&=
- \Phi^\top D_{\mu} (I - \gamma \MKQ^{\pi^{\lambda}_{\theta^\star}})\Phi\Delta_t
+ \mathrm{R}(\Delta_t)\eqsp,
\end{align}
where the remainder terms satisfy
\begin{align}
\|\wR_t(\Delta_t)\|
\leq
\frac{\gamma}{2\lambda}\|\Delta_t\|^2\eqsp,
\qquad
\|\mathrm{R}(\Delta_t)\|
\leq
\frac{\gamma}{2\lambda}\|\Delta_t\|^2\eqsp.
\end{align}
\end{lemma}

\begin{proof}
We prove the first decomposition; the second one is analogous. By the integral form of Taylor's formula,
\begin{align}
F(\theta_t,z_t)-F(\theta^\star,z_t)
=
\nabla F(\theta^\star,z_t)\Delta_t
+
\int_0^1
\Bigl(
\nabla F(\theta^\star + u\Delta_t,z_t)
-
\nabla F(\theta^\star,z_t)
\Bigr)\Delta_t\,du .
\end{align}
Using the explicit Jacobian formula \eqref{eq:jacobian_formula}, we have
\begin{align}
\nabla F(\theta^\star,z_t)\Delta_t
=
-\Phi^\top \Lam_t (I-\gamma\bMKQ_t\Pi_{\theta^\star}^{\lambda})\Phi\Delta_t .
\end{align}
Hence, defining
\begin{align}
\wR_t(\Delta_t)
:=
\int_0^1
\Bigl(
\nabla F(\theta^\star + u\Delta_t,z_t)
-
\nabla F(\theta^\star,z_t)
\Bigr)\Delta_t\,du ,
\end{align}
we obtain the first representation. Moreover, by \Cref{lem:jacobian_smoothness},
\begin{align}
\|\wR_t(\Delta_t)\|
&\leq
\int_0^1
\big\|
\nabla F(\theta^\star + u\Delta_t,z_t)
-
\nabla F(\theta^\star,z_t)
\big\|\,\|\Delta_t\|\,du \\
&\leq
\int_0^1
\frac{\gamma}{\lambda}u\|\Delta_t\|^2\,du
=
\frac{\gamma}{2\lambda}\|\Delta_t\|^2 .
\end{align}
The proof is complete.
\end{proof}
\Cref{lem:Delta_with_remainder} and \Cref{lem:Delta_with_remainder_pr} provide bounds for high-order moments of the remainder terms \(\Rlast_{t+1}\) and \(\Rpr_n\). These estimates are an essential step in the Gaussian approximation analysis.
\begin{lemma}
\label{lem:Delta_with_remainder}
    Assume that \Cref{assum:regularity}--\Cref{assum:features} hold. Then, the last--iterate error admits the decomposition
    \begin{align}
        \Delta_{t+1} = \sum_{j=0}^{t}\alpha_j\Gamma_{j+1:t}  \beps_j^{(0)} + \Rlast_{t+1}\eqsp.
    \end{align}
    Moreover, the remainder satisfies
    \begin{align}
\PE^{1/p}[\|\Rlast_{t+1}\|^p]\leq p^2\operatorname{C}^{\operatorname{last}}\alpha_t\eqsp,
    \end{align}
where \(\operatorname{C}^{\operatorname{last}}\) is given by
\begin{align}
\operatorname{C}^{\operatorname{last}} \asymp \frac{\|\Delta_0\|} {\alpha_0} + \frac{ \sqrt{\Cmoment}}{\kappa\lambda} + \frac{ \taumix \Cmoment^{1/4}}{\sqrt{\kappa}} + \frac{\taumix\sqrt{\Cmoment}\eps_{\max}}{\kappa}\eqsp.
\end{align}
\end{lemma}
\begin{proof}
    Unrolling \eqref{eq:delta_clt_decomposition_3} up to \(t=0\) we immediately obtain the direct formula 
    \begin{align}
    \label{eq:Rlast_def}
         \Rlast_{t+1} &= \Gamma_{0:t}
    \Delta_0
    + 
    \sum_{j=0}^t\alpha_j\Gamma_{j+1:t} \wR_j(\Delta_j) 
    + \sum_{j=0}^{t}\alpha_j\Gamma_{j+1:t}  E_j^{(0)}\Delta_j\\
    &+ \sum_{j=0}^{t}\alpha_j\Gamma_{j+1:t}( \beps_j^{(1)} + E_j^{(1)}\Delta_j)\eqsp.
    \end{align}
The rest of the proof follows the same lines as \Cref{lem:xi_0_bound} and \Cref{lem:xi_1_bound}. The argument is simplified by the fact that high-order moment bounds for the last-iterate error are already available through \Cref{thm:moments}. By \Cref{lem:contraction}, we have \(\|\Gamma_{m:n}\| \le \gprod{m}{n}\), where \(\gprod{m}{n}\) is defined in \eqref{eq:gprod_definition}. Applying Minkowski's inequality to \eqref{eq:Rlast_def}, we obtain 
\begin{align}
    &\PE^{1/p}\big[\|\Rlast_{t+1} \|^{p}\|\big]\leq  \PE^{1/p}\big[\|\Gamma_{0:t}
    \Delta_0\|^p\big]
    + 
    \PE^{1/p}\Big[\big\|\sum_{j=0}^t\alpha_j\Gamma_{j+1:t} \wR_j(\Delta_j) \big\|^p\Big]
     \\
    &\qquad+\PE^{1/p}\Big[\big\|\sum_{j=0}^{t}\alpha_j\Gamma_{j+1:t}  E_j^{(0)}\Delta_j \big\|^p\Big]+ \PE^{1/p}\Big[\big\|\sum_{j=0}^{t}\alpha_j\Gamma_{j+1:t}( \beps_j^{(1)} + E_j^{(1)}\Delta_j)\big\|^p\Big]\eqsp.
\end{align}
We use \Cref{lem:P_alpha_ineq} and obtain \(\alpha_0\gprod{1}{t}\leq \alpha_t\), and hence
\begin{align}
\label{eq:transient_result}
    \PE^{1/p}\big[\|\Gamma_{0:t}
    \Delta_0\|^p\big]\leq \gprod{0}{t}\|\Delta_0\|\leq \alpha_t\alpha_0^{-1}\|\Delta_0\|\eqsp.
\end{align}
For the second term, we apply Minkowski's inequality together with the estimate from \Cref{lem:jacobian_reminders}:
\begin{align}
\label{eq:nonlinear_result}
\PE^{1/p}\Big[\big\|\sum_{j=0}^t\alpha_j\Gamma_{j+1:t} \wR_j(\Delta_j) \big\|^p\Big]
    &\leq \sum_{j=0}^{t}\alpha_j\gprod{j+1}{t}\PE^{1/p}\big[\|\wR_j\|^p\big]\\
    &\leq\frac{\gamma}{2\lambda}\sum_{j=0}^{t}\alpha_j\gprod{j+1}{t}\PE^{1/p}\big[\|\Delta_j\|^{2p}\big]\\
    &\leq\frac{\gamma p^2\sqrt{\Cmoment}}{2\lambda}\sum_{j=0}^{t}\alpha_j^2\gprod{j+1}{t}\leq \frac{8p^2\alpha_t\gamma \sqrt{\Cmoment}}{\kappa\lambda}\eqsp,
\end{align}
where we additionally used \Cref{thm:moments} together with \Cref{lem:rate_of_convergence}.  By construction, the sequence \(\alpha_j\Gamma_{j+1:t}  E_j^{(0)}\Delta_j\) forms a martingale difference sequence. Burkholder’s inequality \cite[Theorem 8.1]{osekowski} gives
    \begin{align} 
    \label{eq:brch_square}\PE^{2/p}\Big[\big\|\sum_{j=0}^{t}\alpha_j\Gamma_{j+1:t}  E_j^{(0)}\Delta_j \big\|^p\Big]
    &\leq p^2\sum_{j=0}^t           \PE^{2/p}\Big[\big\|\alpha_j\Gamma_{j+1:t}  E_j^{(0)}\Delta_j \big\|^p\Big]\\
    &\leq p^2 \sum_{j=0}^t          \alpha_j^2\gprod{j+1}{t}^2\PE^{2/p}\Big[\big\|  E_j^{(0)}\big\|\big\|\Delta_j \big\|^p\Big]\\
    &\overset{(a)}{\leq} 18^2\taumix^2 p^2 \sum_{j=0}^t          \alpha_j^2\gprod{j+1}{t}^2\PE^{2/p}\Big[\big\|\Delta_j \big\|^p\Big]\\
    &\overset{(b)}{\leq} 18^2\taumix^2 p^4\sqrt{\Cmoment} \sum_{j=0}^t          \alpha_j^3\gprod{j+1}{t}^2\\
    &\overset{(c)}{\leq} \frac{72^2 \alpha_t^2\taumix^2 p^4\sqrt{\Cmoment}}{\kappa}\eqsp,
    \end{align}
where \((a)\) follows from \Cref{lem:E_bond}, \((b)\) from \Cref{thm:moments}, and \((c)\) from \Cref{lem:rate_of_convergence}. Taking the square root in \eqref{eq:brch_square}, we obtain
\begin{align}
    \label{eq:mart_term_result}
    \PE^{1/p}\Big[\big\|\sum_{j=0}^{t}\alpha_j\Gamma_{j+1:t}  E_j^{(0)}\Delta_j \big\|^p\Big]
    &\leq \frac{72 \alpha_t\taumix p^2\Cmoment^{1/4}}{\sqrt{\kappa}}\eqsp.
\end{align}
The estimation of the Markovian remainder follows the same ideas as in \Cref{lem:xi_1_bound}. The full proof is deferred to \Cref{lem:markovian_reminder_clt}; here we state only the resulting bound:
\begin{align}
    \label{eq:mark_term_result}\PE^{1/p}\Big[\big\|\sum_{j=0}^{t}\alpha_j\Gamma_{j+1:t}( \beps_j^{(1)} + E_j^{(1)}\Delta_j)\big\|^p\Big]\lesssim \frac{\alpha_t\taumix\sqrt{\Cmoment}\eps_{\max}}{\kappa}\eqsp.
\end{align}
The claim follows by combining \eqref{eq:transient_result}, \eqref{eq:nonlinear_result}, \eqref{eq:mart_term_result}, and \eqref{eq:mark_term_result}.
\end{proof}

\begin{lemma}
 Assume that \Cref{assum:regularity}--\Cref{assum:features} hold. Then, the Polyak--Ruppert remainder satisfies 
\label{lem:Delta_with_remainder_pr}
    \[\PE^{1/p}\big[\|\Rpr_n\|^p\big]\lesssim p^2\operatorname{C}^{\operatorname{last}}(1-\omega)^{-1}  n^{1/2-\omega}\eqsp.\]
\end{lemma}
\begin{proof}
Recall \eqref{eq: W_n definition}:
\begin{align}
\mathrm{R}_n^{\operatorname{pr}} =  \frac{1}{\sqrt{n}}\sum_{t=1}^n\mathrm{R}_t^{\operatorname{last}}\eqsp.
\end{align}
 The proof follows directly from the definition of \(\Rpr_n\) and Minkowski's inequality:
\begin{align}
\PE^{1/p}\big[\|\Rpr_n\|^p\big] &= \frac{1}{\sqrt{n}}\PE^{1/p}\Big [\big\|\sum_{t=1}^n\mathrm{R}_t^{\operatorname{last}}\big\|^p\Big]\\
&\leq\frac{1}{\sqrt{n}}\sum_{t=1}^n\PE^{1/p}\Big [\big\|\mathrm{R}_t^{\operatorname{last}}\big\|^p\Big]\\
&\leq \frac{p^2}{\sqrt{n}}\operatorname{C}^{\operatorname{last}}\sum_{t=1}^n\alpha_t\\
&\lesssim p^2\operatorname{C}^{\operatorname{last}}(1-\omega)^{-1} n^{1/2-\omega}\eqsp,
\end{align}
where the last inequality uses \(\alpha_t = c_0(t+k_0)^{-\omega}\) with \(\omega\in(1/2,1)\).
\end{proof}

\begin{lemma}
\label{lem:E_bond}
Assume that \Cref{assum:regularity}--\Cref{assum:features} hold. Then,
\begin{align}
\|E_j^{(0)}\|
\leq
18\taumix,
\qquad
\|\beps_j^{(0)}\|
\leq
9\taumix \eps_{\max}\eqsp.
\end{align}
\end{lemma}

\begin{proof}
We first bound \(E(z)\) and \(\beps(z)\) before the Poisson decomposition. The estimate
\(
\|\beps(z)\|
\leq
\eps_{\max}
\)
was established in \Cref{lem:xi_t_bound}.
Next, using \Cref{assum:features}, we obtain the uniform Jacobian bound
\begin{align}
\label{eq:jacobian_at_optioma}
\|\nabla F(\theta^\star,z)\|
&=
\big\|\Phi^\top \Lam(z)\big(I-\gamma \bMKQ(z)\Pi^{\lambda}_{\theta^\star}\big)\Phi\big\| \\
&\leq
\big\|\Phi^\top \Lam(z)\Phi\big\|
+
\gamma
\big\|\Phi^\top \Lam(z)\bMKQ(z)\Pi^{\lambda}_{\theta^\star}\Phi\big\| \\
&=
\big\|\phi(s,a)\phi(s,a)^\top\big\|
+
\gamma
\big\|
\phi(s,a)
\sum_{a'\in\cA}\pi^\lambda_{\theta^\star}(a'\mid s')
\phi(s',a')^\top
\big\| \\
&\leq
\|\phi(s,a)\|^2
+
\gamma \|\phi(s,a)\|
\sum_{a'\in\cA}\pi^\lambda_{\theta^\star}(a'\mid s')
\|\phi(s',a')\| \\
&\leq
1+\gamma \leq 2\eqsp.
\end{align}

Therefore,
\(
\|E(z)\|
=
\|\nabla F(\theta^\star,z)-\nabla \bar F(\theta^\star)\|
\leq
4.
\)
Applying \Cref{lem:poisson_transform_uniform_bound}, we obtain
\[
\|\pois{E}{}(z)\|
\leq
\frac{16}{3}\taumix.
\]
Hence,
\begin{align}
\|\totMKQ \pois{E}{}(z)\|
&\leq
\|\pois{E}{}(z)\|
+
\|E(z)\| \leq
\frac{16}{3}\taumix + 4.
\end{align}
Recalling the martingale decomposition
\(
E_j^{(0)}
=
\pois{E}{}(z_j)
-
\totMKQ \pois{E}{}(z_{j-1}),
\)
we conclude that
\begin{align}
\|E_j^{(0)}\|
&\leq
\|\pois{E}{}(z_j)\|
+
\|\totMKQ \pois{E}{}(z_{j-1})\| \leq
2\cdot \frac{16}{3}\taumix + 4
\leq
18\taumix\eqsp.
\end{align}

The bound for \(\beps_j^{(0)}\) is derived analogously.
\end{proof}

\begin{lemma}
\label{lem:markovian_reminder_clt}
Assume that \Cref{assum:regularity}--\Cref{assum:features} hold. Then,
\begin{align}
    \PE^{1/p}\Big[\big\|\sum_{j=0}^{t}\alpha_j\Gamma_{j+1:t}( \beps_j^{(1)} + E_j^{(1)}\Delta_j)\big\|^p\Big] \lesssim \frac{\alpha_t\taumix\sqrt{\Cmoment}\eps_{\max}}{\kappa}\eqsp.
\end{align}
\end{lemma}

\begin{proof}
We split the Markovian remainder into two parts
\begin{align}
S_t^{\beps}
&:=
\sum_{j=0}^{t}
\alpha_j\Gamma_{j+1:t}\beps_j^{(1)}\eqsp,\quad
S_t^{E}=
\sum_{j=0}^{t}
\alpha_j\Gamma_{j+1:t}E_j^{(1)}\Delta_j \eqsp.
\end{align}
By Minkowski's inequality, it is enough to bound \(S_t^{\beps}\) and \(S_t^{E}\) separately.
For \(S_t^{\beps}\), we use the Poisson decomposition
\(
\beps_j^{(1)}
=
\totMKQ \pois{\beps}{j-1}
-
\totMKQ \pois{\beps}{j}
\)
and apply the discrete integration-by-parts identity in matrix form from \Cref{lem:discret_ibp}, with
\(
u_j = \alpha_j\Gamma_{j+1:t}
\)
and
\(
A_j = \totMKQ \pois{\beps}{j}
\). This gives
\begin{align}
S_t^{\beps}
&=
-\alpha_t\totMKQ \pois{\beps}{t}
+
\sum_{j=0}^{t-1}
\big(
\alpha_{j+1}\Gamma_{j+2:t}
-
\alpha_j\Gamma_{j+1:t}
\big)
\totMKQ \pois{\beps}{j}.
\end{align}
We now analyze the coefficient term:
\begin{align}
\big\| \alpha_j\Gamma_{j+1:t} - \alpha_{j+1}\Gamma_{j+2:t}\big\|
&=
\big\|\alpha_j
\big(I -\alpha_{j+1}\Phi^\top D_{\mu}(I - \gamma\MKQ^{\pi^{\lambda}_{\theta^\star}})\Phi\big)
\Gamma_{j+2:t}
-
\alpha_{j+1}\Gamma_{j+2:t}\big\|\\
&\leq
\big\|
(\alpha_j - \alpha_{j+1})I
-
\alpha_j\alpha_{j+1}
\Phi^\top D_{\mu}(I - \gamma\MKQ^{\pi^{\lambda}_{\theta^\star}})\Phi
\big\|
\big\|\Gamma_{j+2:t}\big\|\\
&\leq 3\alpha_j^2\gprod{j+2}{t} \eqsp,
\end{align}
where we used \Cref{lem:alpha_delta} and estimation \eqref{eq:jacobian_at_optioma}. Therefore,
\begin{align}
\PE^{1/p}\big[\|S_t^{\beps}\|^p\big]
&\lesssim
\taumix\eps_{\max}\alpha_t
+
\taumix\eps_{\max}
\sum_{j=0}^{t-1}
\alpha_j^2\gprod{j+2}{t} 
\lesssim
\frac{\alpha_t\taumix\eps_{\max}}{\kappa}
\eqsp.
\end{align}

We now consider \(S_t^{E}\). Using the Poisson decomposition
\(
E_j^{(1)}
=
\totMKQ \pois{E}{j-1}
-
\totMKQ \pois{E}{j}
\)
and applying \Cref{lem:discret_ibp} with
\(
u_j = \alpha_j\Gamma_{j+1:t}\Delta_j
\)
and
\(
A_j = \totMKQ \pois{E}{j}
\),
we obtain
\begin{align}
S_t^{E}
&=
-\alpha_t\totMKQ \pois{E}{t}\Delta_t 
+
\sum_{j=0}^{t-1}
\big(
\alpha_{j+1}\Gamma_{j+2:t}
-
\alpha_j\Gamma_{j+1:t}
\big)
\totMKQ \pois{E}{j}\Delta_j \\
&\quad+
\sum_{j=0}^{t-1}
\alpha_{j+1}\Gamma_{j+2:t}
\totMKQ \pois{E}{j}
(\Delta_{j+1}-\Delta_j).
\end{align}
Using \Cref{lem:E_bond}, \Cref{thm:moments}, and the coefficient bound above, the first two terms are bounded as
\begin{align}
&\alpha_t\PE^{1/p}\big[\|\totMKQ \pois{E}{t}\Delta_t\|^p\big]
+
\PE^{1/p}\Big[
\big\|
\sum_{j=0}^{t-1}
\big(
\alpha_{j+1}\Gamma_{j+2:t}
-
\alpha_j\Gamma_{j+1:t}
\big)
\totMKQ \pois{E}{j}\Delta_j
\big\|^p
\Big] \\
&\quad\lesssim
\taumix\alpha_t
\PE^{1/p}\big[\|\Delta_t\|^p\big]
+
\taumix
\sum_{j=0}^{t-1}
\alpha_j^2\gprod{j+2}{t}
\PE^{1/p}\big[\|\Delta_j\|^p\big] \\
&\quad\lesssim
\taumix\sqrt{\Cmoment}\alpha_t\sqrt{\alpha_t}
+
\taumix\sqrt{\Cmoment}
\sum_{j=0}^{t-1}
\alpha_j^2\gprod{j+2}{t}\sqrt{\alpha_j}
\leq \frac{\alpha_t\taumix\sqrt{\Cmoment}}{\kappa}\eqsp.
\end{align}
For the last term, recalling from \eqref{eq:Delta_Delt} that
\[
    \|\Delta_{j+1} - \Delta_j\| \leq 3\alpha_j\|\Delta_j\| + \alpha_j\eps_{\max}\eqsp,
\]
we get
\begin{align}
&\PE^{1/p}\Big[
\big\|
\sum_{j=0}^{t-1}
\alpha_{j+1}\Gamma_{j+2:t}
\totMKQ \pois{E}{j}
(\Delta_{j+1}-\Delta_j)
\big\|^p
\Big] \\
&\quad\lesssim
\taumix
\sum_{j=0}^{t-1}
\alpha_{j+1}\gprod{j+2}{t}
\PE^{1/p}\big[\|\Delta_{j+1}-\Delta_j\|^p\big] \\
&\quad\lesssim
\taumix
\sum_{j=0}^{t-1}
\alpha_j^2\gprod{j+2}{t}
\PE^{1/p}\big[\|\Delta_j\|^p\big]
+
\taumix\eps_{\max}
\sum_{j=0}^{t-1}
\alpha_j^2\gprod{j+2}{t} \\
&\quad\lesssim
\taumix\sqrt{\Cmoment}
\sum_{j=0}^{t-1}
\alpha_j^2\gprod{j+2}{t}\sqrt{\alpha_j}
+
\taumix\eps_{\max}
\sum_{j=0}^{t-1}
\alpha_j^2\gprod{j+2}{t}
\leq \frac{\alpha_t\taumix\sqrt{\Cmoment}\eps_{\max}}{\kappa}\eqsp.
\end{align}
Combining the bounds for \(S_t^{\beps}\) and \(S_t^E\) completes the proof.
\end{proof}

\subsection{One-Step Contraction Analysis}
\begin{lemma}\label{lem:contraction}
Assume that \Cref{assum:regularity}--\Cref{assum:features} hold. Then, for any \(\alpha\in(0, {\kappa/8})\), 
\begin{align}
\| I - \alpha\, \Phi^\top D_\mu (I - \gamma \MKQ^{\pi^{\lambda}_{\theta^\star}}) \Phi \|
\leq 1 - \alpha\kappa/4\eqsp.
\end{align}
\end{lemma}
\begin{proof}
Consider the following measure on \(\mathcal{Z}\times\cS\):
\[
(s,a,s') \sim \mu\eqsp, \quad a' \sim \pi^{\lambda}_{\theta^\star}(\,\cdot \mid s')\eqsp.
\]
In what follows, expectations are taken with respect to this measure. Recall that in \eqref{eq:G_definition} we introduced \(G= \Phi^\top D_\mu (I - \gamma \MKQ^{\pi^{\lambda}_{\theta^\star}}) \Phi\). Begin with the following representation:
\begin{align}
\label{eq:G_prob_repr}
    G = \Phi^\top D_\mu(I-\gamma \MKQ^{\pi^{\lambda}_{\theta^\star}}) \Phi = \PE [\phi(s,a)\left(\phi(s,a) - \gamma\phi(s',a')\right)^\top]\eqsp.
\end{align}
Rather than directly estimating $\|I - \alpha G\|$, we control the quantity \(\|(I - \alpha G)^\top(I-\alpha G) \|\).   Note that 
\begin{align}\label{eq:open_brackets}
0\preceq (I - \alpha G)^\top(I-\alpha G) = I - \alpha(G + G^\top - \alpha G^\top G)= I - \alpha B\eqsp.
\end{align}
We now derive a lower bound on $G + G^\top$. By \eqref{eq:G_prob_repr}, we have
\begin{align}
G + G^\top
&=
\PE\left[\phi(s, a)\bigl(\phi(s,a)-\gamma \phi(s', a')\bigr)^\top
+
\bigl(\phi(s,a)-\gamma \phi(s', a')\bigr)\phi(s,a)^\top \right]\\
&=
\PE\left[2\phi(s,a)\phi(s,a)^\top
-
\gamma\Bigl(
\phi(s,a)\phi(s',a')^\top
+
\phi(s',a')\phi(s,a)^\top
\Bigr) \right]\\
&\succeq  \PE\left[\phi(s,a)\phi(s,a)^\top\right] - \gamma^2 \PE[\phi(s',a')\phi(s',a')^\top]\\
&= \Sigma_{\phi}^{\pi_b} - \gamma^2\Sigma_{\phi}^{{\theta^\star}}\eqsp.
\end{align}
where we used an elementary inequality 
\begin{align}\label{eq:vector_cauchy}
\gamma(uv^\top + vu^\top)
\;\preceq\;
\gamma^2uu^\top + vv^\top,
\qquad \forall u,v \in \mathbb{R}^d\eqsp.
\end{align}
Using the representation \eqref{eq:G_prob_repr}, we obtain
\begin{align}
    \|G\|&= \big\|\PE [\phi(s,a)\left(\phi(s,a) - \gamma\phi(s',a')\right)^\top]\big\|\\
    &\leq \PE[\|\phi(s,a)\left(\phi(s,a) - \gamma\phi(s',a')\right)^\top\|]\\
    &\leq \PE[\|\phi(s,a)\phi(s,a)^\top\|] + \gamma \PE[\|\phi(s,a)\phi(s',a')^\top\|]\\
    &\leq 1 + \gamma\eqsp,
\end{align}
where the last inequality uses \(\|\phi(s,a)\|\leq 1\). Hence, \(\|G^\top G\|\leq 4\),  and therefore
\begin{align}
    0&\preceq (I - \alpha G)^\top(I-\alpha G)\\
    &\preceq  I - \alpha\big(\Sigma_{\phi}^{\pi_b} - \gamma^2\Sigma_{\phi}^{{\theta^\star}}\big) +4\alpha^2 I \\
    &\overset{(a)}{\preceq} I - \alpha\kappa I+4\alpha^2 I \\
    &\overset{(b)}{\preceq} I - \alpha\frac{\kappa}{2}I\eqsp,
\end{align}
where \((a)\) follows from \Cref{assum:features}, and \((b)\) from the restriction \(\alpha\in(0, \kappa/8)\). Finally,
\begin{align}
   \|I - \alpha G\| = \|(I - \alpha G)^\top(I-\alpha G) \|^{1/2}\leq \sqrt{1-\alpha\kappa/2} \leq 1 - \alpha\kappa/4\eqsp.
\end{align}

\end{proof}

\section{Gaussian Comparison and Anti-concentration}
\begin{proposition}[Proposition 1 in \cite{sheshukova2025gaussian}]
\label{lem:shao}
Let \(\nu\) be the standard Gaussian measure in \(\mathbb{R}^d\). Then for any random vectors \(X, Y\) taking values in \(\mathbb{R}^d\), and any \(p \ge 1\),
\begin{align}
\sup_{B \in \mathcal{C}(\mathbb{R}^d)}
\big|
\mathbb{P}(X + Y \in B) - \nu(B)
\big|
&\le
\sup_{B \in \mathcal{C}(\mathbb{R}^d)}
\big|
\mathbb{P}(X \in B) - \nu(B)
\big| 
+ 2 c_d^{\frac{p}{p+1}}
\mathbb{E}^{\frac{1}{p+1}}\!\left[\|Y\|^p\right],
\end{align}
where \(c_d\) is the isoperimetric constant of the class \(\mathcal{C}(\mathbb{R}^d)\) \cite{ball1993reverse}.
\end{proposition}
 For the class of convex sets \(\mathcal{C}(\mathbb{R}^d)\), the isoperimetric constant satisfies
 \begin{align}
     e^{-1}\log^{1/2} d\leq c_d\leq 4d^{1/4}\eqsp.
 \end{align}
\begin{lemma}[Lemma 1 in \cite{sheshukova2025gaussian}]
\label{lem:convex_dist}
Assume \Cref{assum:regularity}--\Cref{assum:features}. Let \(Z \sim \mathcal{N}(0, I_d)\). Then the convex distance between the distributions of \(\Sigma_n^{1/2} Z\) and \(\Sigma_\infty^{1/2} Z\) satisfies
\begin{align}
\dconv\left(\Sigma_n^{1/2} Z, \Sigma_\infty^{1/2} Z\right)
\le
C_\infty n^{\omega - 1} \eqsp.
\end{align}
\end{lemma}
\begin{proof}
The proof is identical to that of \cite[Appendix D.1]{sheshukova2025gaussian} and is therefore omitted.
\end{proof}
\begin{lemma}[Lemma 3 in \cite{sheshukova2025gaussian}]
\label{lem_marina_top}
Assume \Cref{assum:regularity}--\Cref{assum:features} hold. Then for any \(i \in \{0,\dots,n-1\}\),
\begin{align}
\lambda_{\max}(Q_i) \le C_Q \eqsp.
\end{align}
 Moreover,
\begin{align}
\lambda_{\min}(Q_i) \ge C_Q^{\min},
\qquad
\|\Sigma_n^{-1/2}\| \le C_\Sigma \eqsp,
\end{align}
where the matrix \(\Sigma_n\) is defined in \eqref{eq:Sigma_n_def}.
\end{lemma}
\begin{lemma}
\label{lem:clt:application}
Assume that \Cref{assum:regularity}--\Cref{assum:features} hold. Then, 
\begin{align}
    \dconv (W_n,\,\Sigmabf_n^{1/2}Z) 
    \lesssim
C_W\frac{\sqrt{d}\log n }{n^{1/4}} \eqsp,
\end{align}
where 
\begin{align}
    \operatorname{C}_W= \Bigl\{
M
(1-\gamma_{ps})^{-\frac{1}{4}}
\log^{\frac{1}{4}}
(
d 
)
\|\Sigma_n\|_{\mathrm{F}}^{\frac{1}{2}}
+
\sqrt{M}\,
\log^{\frac{1}{2}}\!\bigl(d\|\Sigma_n\|\bigr)\,
\|\Sigma_n\|_{\mathrm{F}}^{\frac{1}{4}}
\Biggr\}\eqsp, \quad M=\taumix\eps_{\max}C_Q\eqsp.
\end{align}
\end{lemma}
\begin{proof}
Under \Cref{assum:UGE}, \cite[Proposition~3.4]{paulin2015concentration} implies that the Markov kernel \(\totMKQ\) admits a pseudo spectral gap \(\gamma_{ps}\) satisfying
\[
\gamma_{ps}\geq \frac{1}{2\taumix}\eqsp.
\]
By \Cref{lem_marina_top} and \Cref{lem:poisson_transform_uniform_bound}, we have the uniform bound
\[
\|Q_j\beps_j^{(0)}\|
\lesssim
\taumix\eps_{\max}C_Q .
\]
The claim then follows directly from \cite[Corollary~4]{wu2025uncertainty} with \(M=\taumix\eps_{\max}C_Q\).
\end{proof}
\section{Auxiliary Lemmas}
\subsection{Properties of step-size sequence}
\label{appendix:step_size}
For the proofs of \Cref{lem:rate_of_convergence}--\Cref{lem:P_alpha_ineq}, we refer the reader to
Appendix~I of \cite{samsonov2026statistical}.

\begin{lemma}\label{lem:rate_of_convergence}
Let $b, c_0 > 0$ and \(\alpha_k = \frac{c_0}{(k + k_0)^{\omega}}\) with \(\omega \in (\frac{1}{2},1)\), $k_0 \geq 0$. Assume that $bc_0 \leq\frac{1}{2}$ and $k_0^{1-\omega} \geq 8/(bc_0)$. 
Then for any $q \in (1;3]$, it holds that 
\[
\sum_{j=1}^{k} \alpha_j^{q} 
\prod_{\ell=j+1}^{k} (1 - \alpha_\ell b)
\leq
\frac{4}{b}\,\alpha_k^{q-1}.
\]
\end{lemma}

\begin{lemma}\label{lem:alpha_delta}
Let $b, c_0 > 0$ and \(\alpha_k = \frac{c_0}{(k + k_0)^{\omega}}\) with \(\omega \in (\frac{1}{2},1)\), $k_0 \geq 0$. Assume that $bc_0 \leq\frac{1}{2}$ and $k_0^{1-\omega}\geq 8/(bc_0)$. Then it holds that
\[
    \frac{\alpha_k}{\alpha_{k+1}} \le 1 + b \alpha_{k+1}.
\]
\end{lemma}

\begin{lemma}\label{lem:P_alpha_ineq}
Let $b, c_0 > 0$ and \(\alpha_k = \frac{c_0}{(k + k_0)^{\omega}}\) with \(\omega \in (\frac{1}{2},1)\), $k_0 \geq 0$. Assume that $bc_0 \leq\frac{1}{2}$ and $k_0^{1-\omega}\geq 8/(bc_0)$. Then it holds that
\[
    \alpha_j \prod_{l=j+1}^{k} (1 - \alpha_l b) \le \alpha_k \eqsp.
\]
\end{lemma}

\subsection{Matrix form of the projected soft Bellman equation}
\begin{lemma}
\label{lem:projected_bellman_matrix_form}
Assume \Cref{assum:features}. Let \(\theta^\star\in\mathbb{R}^d\) be a solution of the projected soft Bellman equation \eqref{eq:projected_bellman}.
Then \(\theta^\star\) satisfies
\begin{align}\label{eq:theta_star_normal_equation}
\Phi^\top D_\mu
\left(
r+\gamma \MKQ V_{\theta^\star}^{\lambda}
-\Phi\theta^\star
\right)
=0.
\end{align}
\end{lemma}

\begin{proof}
Since \(Q_{\theta^\star}=\Phi\theta^\star\in\mathrm{span}(\Phi)\), we have
\(
\Pi_\mu Q_{\theta^\star}=Q_{\theta^\star}.
\)
Therefore, \eqref{eq:projected_bellman} is equivalent to
\(
Q_{\theta^\star}
=
\Pi_\mu \mathcal{T}_\lambda Q_{\theta^\star}.
\)
Moreover, the \(D_\mu\)-orthogonal projection onto \(\mathrm{span}(\Phi)\) is given by
\begin{equation}\label{eq:projection_operator}
\Pi_\mu
=
\Phi(\Phi^\top D_\mu\Phi)^{-1}\Phi^\top D_\mu.
\end{equation}
Hence, by definition of the soft Bellman equation and \eqref{eq:projection_operator}
\[
\Phi\theta^\star
=
\Phi(\Phi^\top D_\mu\Phi)^{-1}
\Phi^\top D_\mu
\left(
r+\gamma \MKQ V_{\theta^\star}^{\lambda}
\right).
\]
Since \(\Phi^\top D_\mu\Phi\) is invertible, the columns of \(\Phi\) are linearly independent in the \(D_\mu\)-inner product. Thus the coordinates in the representation over \(\mathrm{span}(\Phi)\) are unique, and we obtain
\[
\theta^\star
=
\bigl(\Phi^\top D_\mu\Phi\bigr)^{-1}
\Phi^\top D_\mu
\left(
r+\gamma \MKQ V_{\theta^\star}^{\lambda}
\right).
\]
Multiplying both sides by \(\Phi^\top D_\mu\Phi\) gives
\[
\Phi^\top D_\mu\Phi\,\theta^\star
=
\Phi^\top D_\mu
\left(
r+\gamma \MKQ V_{\theta^\star}^{\lambda}
\right),
\]
which completes the proof.
\end{proof}
\end{document}